\documentclass{article}

\usepackage{microtype}
\usepackage{graphicx}
\usepackage{graphicx}
\graphicspath{{figures/}}
\usepackage{subcaption}
\usepackage{booktabs} 

\usepackage{hyperref}

\usepackage[preprint]{icml2026}

\renewcommand{\today}{}

\usepackage{amsmath}
\usepackage{amssymb}
\usepackage{mathtools}
\usepackage{amsthm}
\usepackage{bm}

\usepackage[capitalize,noabbrev]{cleveref}

\theoremstyle{plain}
\newtheorem{theorem}{Theorem}[section]

\newtheorem{lemma}[theorem]{Lemma}

\theoremstyle{definition}

\newtheorem{assumption}[theorem]{Assumption}
\theoremstyle{remark}
\newtheorem{remark}[theorem]{Remark}

 \usepackage[disable,textsize=tiny]{todonotes}

\icmltitlerunning{How Label Imbalance Shapes Geometry: A General Spectral Analysis of Multi-Label Neural Collapse}

\begin{document}

\twocolumn[
  \icmltitle{How Label Imbalance Shapes Geometry:\\ A General Spectral Analysis of Multi-Label Neural Collapse}







  \begin{icmlauthorlist}
    \icmlauthor{Xiaoxuan Ma}{yyy}
    \icmlauthor{Yixuan Yang}{yyy}
    \icmlauthor{Song Li}{yyy}
    \icmlauthor{Xiangyun Hui}{yyy}

  \end{icmlauthorlist}

  \icmlaffiliation{yyy}{Department of Mathematics, Zhejiang University, Hangzhou, China}

  \icmlcorrespondingauthor{Song Li}{songli@zju.edu.cn}

  \icmlkeywords{Machine Learning, ICML}

  \vskip 0.3in
]



\printAffiliationsAndNotice{}  


\begin{abstract}
This work investigates the phenomenon of Neural Collapse (NC) in multi-label classification, extending its conceptual framework from multi-class learning to general correlated and imbalanced multi-label settings.
Although recent studies have identified a ``tag-wise averaging'' structure for multi-label features, this view relies on implicit assumptions of label balance and combinatorial symmetry. 
Consequently, it fails to account for the geometrical distortions caused by intrinsic label correlations and data imbalance, which are common in practice. 
We resolve the multiplicity-one imbalance conjecture raised by \citet{li2023neural}, showing that higher-multiplicity prototypes obey a class-frequency-weighted synthesis rule rather than uniform averaging.
To address this, we propose a rigorous spectral-control framework to analyze the terminal phase of multi-label learning under general imbalanced conditions. 
We introduce the label covariance spectrum $\kappa_m$, a scalar controlling the distribution-dependent lower-bound geometry, derived from the second-order moment matrix of the label distribution. 
Contrary to the averaging perspective, our analysis reveals that the centered label covariance spectrum controls the stability of terminal geometry by quantifying the weakest centered inter-class contrast directions. 
We prove that the classical Tag-wise Averaging emerges only as a special case under perfect orthogonality. 
Numerical experiments on synthetic distributions validate our theoretical bounds. 
This work resolves the scaled-average aspect of the imbalance conjecture and establishes a unifying theoretical framework that extends Neural Collapse to complex, imbalanced multi-label settings.
\end{abstract}

\section{Introduction}

Deep neural networks have achieved remarkable success in classification, thanks to both strong predictive performance and structured, interpretable representations \cite{LeCun2015, Bengio2013}. A key breakthrough is the Neural Collapse (NC) phenomenon \cite{Papyan2020, Han2022}. In the terminal phase of training over-parameterized networks on balanced multi-class data, last-layer features and classifiers converge to a symmetric, optimal geometry. 
Concretely, within-class features collapse to their means (NC1), class means form a maximally separated Simplex Equiangular Tight Frame (Simplex ETF) (NC2), classifiers become dual to the class means (NC3), and decisions reduce to nearest class-center classification (NC4).
 This structure is empirically pervasive and, under simplified models such as the Unconstrained Features Model (UFM), is theoretically characterized as a global optimum of common losses including cross-entropy (CE) and mean squared error (MSE) \cite{Zhu2021, Zhou2022a, Mixon2022}. Since then, NC has become an essential element in analyzing feature learning, influencing areas from transfer learning to robustness \cite{Galanti2022b, Li2022, Papyan2020, Ji2022}.

Recent research has taken two directions to generalize NC beyond its original idealized setting. The first direction extends NC to  multi-label classification, where an instance can belong to multiple classes. A recent work \cite{li2023neural} established that a generalized NC emerges under the  ``pick-all-labels'' (PAL) formulation. Crucially, it identified a ``tag-wise average'' structure: while single-label features form a Simplex ETF, the mean of a multi-label feature (e.g., for a sample with tags A and B) is the scaled average of the corresponding single-label class means (i.e., the means of class A and class B). This provides the first geometric blueprint for multi-label representation learning.

The second direction investigates NC under class imbalance, a universal challenge in real-world data. Theoretical studies \cite{Fang2021, Thrampoulidis2022, Dang2023} have shown that in imbalanced multi-class settings, the perfect ETF geometry is distorted. Class means no longer share equal norms or angles. Instead, their lengths become dependent on class frequencies, and classifier norms scale accordingly. In extreme imbalance, Minority Collapse can occur, where minority-class features and classifiers become indistinguishable \cite{Fang2021}. Concurrent analyses of the cross-entropy (CE) loss under the unconstrained feature model (UFM) further characterize the resulting block-orthogonal structures, highlighting the sensitivity of NC geometry to the underlying data distribution \cite{Hong2023, Behnia2023}.

Despite these advances, a critical gap remains at their intersection. The existing multi-label NC theory \cite{li2023neural} relies on the implicit assumptions of label balance and combinatorial symmetry---effectively treating multi-label combinations as independent and uniformly distributed. This overlooks the pervasive \textit{intrinsic label correlations} (e.g., ``sky'' and ``cloud'' often co-occur) and \textit{data imbalance} across different label combinations, which are fundamental characteristics of real-world multi-label datasets \cite{Menon2019, Liu2021}. Consequently, the elegant ``tag-wise average'' principle, derived under idealized independence, is insufficient to describe the geometric structures that form under general, correlated, and imbalanced label distributions. In their ICML 2024 paper \emph{Neural Collapse in Multi-label Learning with Pick-all-label Loss} \cite{li2023neural}, Li et al.\ explicitly highlighted the challenging regime of multiplicity-one imbalance: ``We suspect a more general minority collapse phenomenon would happen'', and further ``we conjecture that the scaled average property will still hold between higher multiplicity features and their multiplicity-1 features''.

The conjecture raised by \citet{li2023neural} is to understand the terminal-phase geometry of multi-label learning under data imbalance, in particular when the Multiplicity-1 training data are imbalanced. They conjectured that, despite potential ``minority collapse'' phenomena in this regime, a scaled average relationship may still persist between higher-multiplicity features and their multiplicity-1 counterparts (see the discussion on ``Dealing with data imbalanced-ness'' in \citet{li2023neural}). 

One of our main contributions is to provide a rigorous resolution of the scaled-average aspect of this conjecture. Moving beyond the idealized assumptions of label balance and combinatorial symmetry, we establish a unified spectral-control framework for analyzing multi-label neural collapse under general label correlation and imbalance. Our theory shows that the terminal-phase structure is not governed by naive tag-wise averaging; instead, it is controlled by the centered label covariance spectrum: the spectral scalar $\kappa_m$ quantifies the weakest centered inter-class contrast directions and modulates the distribution-dependent lower-bound geometry. 

As a consequence, the previously observed ``scaled average'' principle emerges as a verifiable special case of our theory, holding only under specific (near-orthogonality / symmetry) spectral conditions. Under general multiplicity-1 imbalance and correlated label distributions, the geometric structure can deviate systematically from simple averaging, yielding count-dependent distortions consistent with imbalance-induced phenomena observed in practice. In this sense, our results turn the scaled-average conjecture of \citet{li2023neural} into a provable and testable theory with explicit spectral and counting controls.

\textbf{Our Contributions.} This work fills this gap by introducing a novel spectral-control framework for multi-label neural collapse under arbitrary label correlation and imbalance. We move beyond the averaging perspective by identifying the spectral and counting quantities that control the lower-bound geometry and the terminal structural constraints. Our contributions are as follows:

\begin{itemize}
    \item \textbf{A Unifying Spectral Framework.} We propose that the global optimization landscape for multi-label learning is controlled by the \textbf{centered label covariance spectrum}. We introduce a key scalar metric, the label covariance spectrum $\kappa_m$, derived from the second-moment matrix of the label distribution. This metric jointly quantifies the effects of label correlation and imbalance through the weakest centered inter-class contrast directions.
    
    \item \textbf{Spectral-Control Principle.} We theoretically prove that the terminal geometry is not explained by merely averaging single-label prototypes. Instead, the centered label covariance operator and its spectral constant $\kappa_m$ enter the lower-bound analysis and control the stability of the geometric structure. In particular, small $\kappa_m$ indicates a weak or nearly invisible centered contrast direction, making the corresponding spectral guarantees ill-conditioned.
    
    \item \textbf{Resolution of the Scaled-Average Conjecture.} Our results resolve the scaled-average aspect of the multiplicity-one imbalance conjecture raised by \citet{li2023neural}.
    Our analysis demonstrates that the previously observed ``tag-wise average'' \cite{li2023neural} is a special case of our spectral-control framework, occurring only under appropriate orthogonality and symmetry conditions. Under general correlated and imbalanced settings, the structure deviates from uniform averaging: higher-multiplicity prototypes are generated from multiplicity-one directions through explicit count-dependent weights.
    
    \item \textbf{Empirical Validation.} We conduct comprehensive numerical experiments on synthetic data with controlled label correlation and imbalance. The results validate our theoretical predictions, supporting the count-dependent terminal structure and the tightness of the bounds governed by $\kappa_m$.
\end{itemize}

In summary, we develop a unifying spectral-control framework for multi-label Neural Collapse under correlation and imbalance. 
We show that the lower-bound geometry is governed by the centered label covariance spectrum. 
We introduce the scalar $\kappa_m$, which quantifies the weakest centered inter-class contrast direction and captures the effects of correlation and imbalance.
We prove that terminal multi-label geometry is not governed by uniform tag-wise averaging in general. Instead, higher-multiplicity prototypes are generated from multiplicity-one directions through explicit count-dependent weights. 
The classical ``tag-wise average'' structure appears only under suitable label orthogonality and symmetry conditions, and is therefore a special case \cite{li2023neural}.
We also resolve the scaled-average aspect of the multiplicity-one imbalance conjecture raised by \citet{li2023neural}. 
Finally, synthetic experiments support the predicted count-dependent distortions and the tightness of the $\kappa_m$-governed bounds.

\section{Related Works}

\paragraph{Neural Collapse on balanced datasets.}
A substantial theory literature explains Neural Collapse (NC) in balanced multi-class settings via simplified last-layer models,
most notably the unconstrained feature model (UFM) and layer-peeled viewpoints
(e.g., \citet{Papyan2020,LuSteinerberger2020NeuralCollapseCE,Zhu2021,Mixon2022,Zhou2022a,Han2022,Ji2022}).
Representative works include \citet{Papyan2020,Zhu2021} and the MSE-landscape line of \citet{Zhou2022a},
as well as further perspectives and refinements (e.g., \citet{Zhou2022b}).
Related efforts toward more network-faithful analyses and optimization/regularization viewpoints include, e.g.,
\citet{tirer2022extended}, \citet{sukenik2023deep}, \citet{tirer2023perturbation}, \citet{RangamaniBanburskiFahey2022WeightDecay}, \citet{PilanciErgen2020ConvexRegularizers}, \citet{Dang2023}.
We also refer to a recent viewpoint in \citet{KothapalliTirerBruna2023GNN}.

\paragraph{Neural Collapse on imbalanced datasets.}
Under class imbalance, the canonical simplex-ETF geometry can be distorted: NC1 often persists, whereas NC2/NC3 may fail,
and in extreme regimes one may observe Minority Collapse \citep{Fang2021}.
Beyond early layer-peeled observations, several analyses characterize imbalance-dependent geometries under different objectives,
including UFM--SVM and SELI-type structures \citep{Thrampoulidis2022,Behnia2023},
block-structured prediction vectors under CE \citep{Hong2023},
and closed-form geometries and minority-collapse thresholds under ReLU-motivated nonnegative-feature (``UFM+'') settings \citep{dang2024neural}.
Related contrasts include fixed-ETF classifier studies \citep{yang2022inducingneuralcollapseimbalanced}, as well as connections to long-tailed recognition/imbalance practice
(e.g., \citet{Kang2020Decoupling,Cao2019LDAM,Liu2021}).

\paragraph{Related works on multi-label learning.}
Multi-label learning predicts a set of tags. Common training paradigms include reducible formulations such as
binary relevance / one-vs-all (BR/OvA) and pick-all-label (PAL), as well as non-decomposable approaches \citep{dembczynski2012label}
(see also \citet{Menon2019}).
On the theory side, a range of surrogate-consistency and learning-theoretic tools have been developed for multi-label regimes,
including early consistency results \citep{pmlr-v19-gao11a}, reductions \citep{Menon2019},
Fenchel--Young losses \citep{blondel2020learningfenchelyounglosses}, sample-compression schemes \citep{pmlr-v35-samei14,2014Generalizing},
(local) Rademacher complexity \citep{xu2014localrademachercomplexitymultilabel,reeve2020optimisticboundsmultioutputprediction},
and Bayes-optimal prediction rules \citep{2010Bayes}.
On the practical side, modern deep architectures are widely adapted to multi-label and extreme multi-label tasks
(e.g., \citet{chang2020tamingpretrainedtransformersextreme,ridnik2021mldecoderscalableversatileclassification}),
where combination-level imbalance becomes more severe as higher-multiplicity labels are scarce.

\paragraph{Multi-label Neural Collapse under PAL.}
\citet{li2023neural} provides the first geometric blueprint of multi-label NC under PAL:
multiplicity-one (single-label) features retain simplex-ETF geometry, whereas higher-multiplicity prototypes obey a structured
tag-wise aggregation law (a multi-label ETF).
Their appendix further discusses challenging regimes in which multiplicity-one data are imbalanced, potentially leading to more general
minority-collapse phenomena, and connects these regimes to diversity-promoting principles such as MCR$^2$
\citep{yu2020learningdiversediscriminativerepresentations} and ``ReduNet'' \citep{chan2021redunetwhiteboxdeepnetwork}.



\section{Problem Setting}
\subsection{Basic Notation}
\label{sec:basic-notation}

\paragraph{Data and multiplicity groups.}
Let $K\in\mathbb{N}$ be the number of classes and denote $[K]=\{1,\dots,K\}$.
We consider multi-label samples grouped by their label-set cardinality (multiplicity)
$m\in\{1,\dots,M\}$, where we assume:
\(
    1 \le m \le M \le K - 1.
\)
For each $m$, define the family of label sets
\[
\mathcal{S}_m=\{S\subset [K]: |S|=m\}.
\]
Let \(r_{m,S}\in\mathbb{N}\) be the number of training samples whose label set equals
\(S\in\mathcal{S}_m\), \(
\text{for } m=1,\ \text{write } r_{1,k} := r_{1,\{k\}}.
\)
Define
\[
\begin{aligned}
N_m &:= \sum_{S\in\mathcal{S}_m} r_{m,S},
\qquad
N := \sum_{m=1}^M N_m, \\
p_{m,S} &:= \frac{r_{m,S}}{N_m},
\qquad (S\in\mathcal{S}_m).
\end{aligned}
\]

For each label set $S$, let $y_S=\mathbf{1}_S\in\mathbb{R}^K$ be its indicator vector.
For each class $k\in[K]$, define the (group-wise) class count
\[
N_m^k=\sum_{\substack{S\in\mathcal{S}_m\\ k\in S}} r_{m,S}.
\]
We assume that \( N_m^k > 0 \) for all \( k \in [K] \) and all \( m \), meaning that each label \( k \in [K] \) appears at least once in each multiplicity \( m \) group; otherwise, labels with zero counts in a given group can be excluded or treated separately.

\paragraph{Model, logits, empirical risk (no bias).}
Let
\(
W=
\begin{bmatrix}
w_1^\top\\
\vdots\\
w_K^\top
\end{bmatrix}
\in\mathbb{R}^{K\times d},
\)
\(
h_{m,S,i}\in\mathbb{R}^d,
\)
and define the logits
\(
z_{m,S,i}=Wh_{m,S,i}\in\mathbb{R}^K.
\)
Given a pointwise cross-entropy
(CE) loss $\ell(z,y_S)$, the averaged empirical risk is
\[
g(WH)=\frac{1}{N}\sum_{m=1}^M\ \sum_{S\in\mathcal{S}_m}\ \sum_{i=1}^{r_{m,S}}
\ell\!\left(Wh_{m,S,i},\,y_S\right).
\]
Throughout, we take the pointwise loss $\ell(\cdot,\cdot)$ to be the pick-all-labels (PAL) loss, i.e.,
\(
\ell(z,y_S)\equiv \mathcal{L}_{\mathrm{PAL}}(z,y_S).
\)

We recall the standard softmax cross-entropy loss. For logits $z\in\mathbb{R}^K$, define
\(
\mathrm{softmax}(z)_k=\frac{e^{z_k}}{\sum_{j=1}^K e^{z_j}},\qquad k\in[K].
\)
For a label vector $y\in\mathbb{R}^K$ that is a probability distribution (e.g., a one-hot label), the cross-entropy loss is
\[
\ell(z,y)=-\sum_{k=1}^K y_k\log\bigl(\mathrm{softmax}(z)_k\bigr).
\]

We study the regularized objective (without bias)
\begin{equation}
\label{eq:objective}
f(W,H)
=
g(WH)
+\frac{\lambda_W}{2}\|W\|_F^2
+\frac{\lambda_H}{2}\|H\|_F^2,
\end{equation}
\[
\lambda_W,\lambda_H>0.
\]
where
\[
\|H\|_F^2
=
\sum_{m=1}^M \sum_{S\in\mathcal{S}_m} \sum_{i=1}^{r_{m,S}}
\|h_{m,S,i}\|_2^2.
\]
\paragraph{Projection and spectral constant.}
Let $\mathbf{1}\in\mathbb{R}^K$ be the all-ones vector and define the centering projection
\[
\Pi := I - \frac{1}{K}\mathbf{1}\mathbf{1}^\top,
\qquad
\mathrm{range}(\Pi)=\{x\in\mathbb{R}^K:\mathbf{1}^\top x=0\}.
\]
For each $m$, define the (population) covariance matrix of label indicators under $p_m$ by
\[
G_m := \mathbb{E}_{S\sim p_m}\bigl[y_S y_S^\top\bigr].
\]
We define the spectral constant on the centered subspace by
\[
\kappa_m := \lambda_{\min}\Bigl( (\Pi G_m \Pi)\big|_{\mathrm{range}(\Pi)} \Bigr).
\]

\subsection{Assumptions}
\begin{assumption}[Non-degeneracy on the centered subspace]\label{assump:nondeg-centered}
We assume that $\kappa_m>0$ for every $m$.
\end{assumption}

\begin{remark}

It is standard in spectral analyses to work on the centered subspace $\mathrm{range}(\Pi)$ since the $\mathbf{1}$-direction
is intrinsically degenerate after centering. The above condition does not require all label combinations to appear;
it only rules out degenerate co-occurrence structures that permanently split classes into isolated components or make certain
inter-class contrast directions statistically invisible.
Mathematically, $\kappa_m>0$ is equivalent to the uniform quadratic-form lower bound: for all $x\in\mathrm{range}(\Pi)$,
\(
x^\top \Pi G_m\Pi x \ge \kappa_m\|x\|_2^2,
\)
and equivalently to the operator inequality (interpreted on $\mathrm{range}(\Pi)$),
\(
\Pi G_m\Pi \succeq \kappa_m \Pi 
\)
on 
\(
\mathrm{range}(\Pi).
\)
Hence $\kappa_m$ admits the Rayleigh-quotient characterization
\(
\kappa_m=\min_{\substack{x\in\mathrm{range}(\Pi)\\ x\neq 0}}
\frac{x^\top \Pi G_m\Pi x}{\|x\|_2^2},
\)
so $\kappa_m$ quantifies the \emph{worst-case} amount of variance along centered inter-class contrast directions.

\noindent\textbf{Case (i): $\kappa_m$ admits a spectral gap.}

If $\kappa_m$ is bounded away from zero, then for every $x\in\mathrm{range}(\Pi)$,
\(
x^\top \Pi G_m\Pi x \ge \kappa_m\|x\|_2^2,
\)
meaning that \emph{every} centered inter-class direction carries nontrivial second-order fluctuation.
Intuitively, the label co-occurrence statistics are sufficiently rich after centering: no meaningful inter-class contrast
is systematically suppressed.

\noindent\textbf{Case (ii): $\kappa_m$ is small (near-degeneracy).}

If $\kappa_m$ is small (or close to $0$), then there exists a nonzero $x\in\mathrm{range}(\Pi)$ such that
\(
x^\top \Pi G_m\Pi x \approx 0,
\)
indicating an (almost) unidentifiable contrast direction with nearly vanishing fluctuation.
In this regime, spectral bounds involving $1/\kappa_m$ become ill-conditioned, reflecting that the data provide
very limited information along certain centered inter-class directions.

\noindent\textbf{Case (iii): $\kappa_m=0$ (degenerate case).}

The case $\kappa_m=0$ means there exists a nonzero $x\in\mathrm{range}(\Pi)$ with $\Pi G_m\Pi x=0$,
i.e., a truly invisible contrast direction.
A canonical example is when classes are permanently split into two disconnected blocks with no cross-block co-occurrence.
For instance, with $K=4$, suppose that (within multiplicity $m$) label sets only occur within $\{1,2\}$ or within $\{3,4\}$,
but never across the two blocks. Then the block-contrast direction
\(
x=(1,1,-1,-1)^\top,
\)
\(
\mathbf{1}^\top x=0,
\)
belongs to $\mathrm{range}(\Pi)$ and can exhibit (near) zero fluctuation, leading to $\kappa_m=0$ (or $\kappa_m$ very small).
\end{remark}

\begin{remark}
The non-degeneracy condition \( \kappa_m>0 \) is strictly stronger than the basic coverage requirement \( N_m^k>0 \). While \( N_m^k>0 \) ensures that first-order statistics (means) are well-defined, the condition \( \kappa_m>0 \) ensures that second-order co-occurrence structures are not degenerate, preserving the geometry of the label simplex. Specifically, \(N_m^k>0\) guarantees that every label appears in the multiplicity group, whereas \( \kappa_m>0 \) ensures that the second-order variance is non-zero, avoiding situations where some class directions become invisible in the data.
\end{remark}


\section{Main Results}
\begin{theorem}[Lower bound in the form of Lemma~\ref{lem:lemma2_lb}]
\label{thm:lower-bound-lemma2-form}
Under the above problem setting, let $(W,H)$ be a global minimizer of the objective \eqref{eq:objective}

For each $m\in\{1,\dots,M\}$, fix any constant $c_{1,m}>0$ and define
\[
\gamma_{1,m}:=\frac{1}{1+c_{1,m}}\cdot\frac{m}{K-m},
\qquad
\Gamma_2:=\frac{1}{N}\sum_{m=1}^M N_m c_{2,m},
\qquad
\]
\[
\rho=:\|W\|_F^2.
\]
\begin{equation}
\begin{aligned}
c_{2,m} :=\;& \frac{c_{1,m} m}{c_{1,m} + 1}\log(m)
+ \frac{m c_{1,m}}{1 + c_{1,m}}
\log\!\left(\frac{c_{1,m} + 1}{c_{1,m}}\right) \\
&\quad + \frac{m}{c_{1,m} + 1}
\log\!\left((K - m)(c_{1,m} + 1)\right).
\end{aligned}
\label{eq:2}
\end{equation}

\[
A_m
:=
\sqrt{
\frac{1}{\kappa_m}\cdot\frac{m(K-m)}{K}\cdot
\Bigg[
\frac{2K}{N_m}
+
\frac{2K^2}{m^2}\max_{S\in\mathcal{S}_m}\sum_{j=1}^K\frac{\mathbf{1}_S(j)}{N_m^j}
\Bigg]
}.
\]
Then the following lower bound holds:
\[
g(WH)-\Gamma_2
\ge
-\frac{1}{N}\sum_{m=1}^M N_m\,\gamma_{1,m}\,A_m\,
\sqrt{\frac{\lambda_W}{\lambda_H}}\,\rho.
\]\label{eq:thm3.1}
\end{theorem}

This theorem is the formal statement of Lemma~\ref{lem:lemma2_lb}; its detailed proof can be found in Appendix~\ref{appendix1} (see Lemma~\ref{lem:lemma2_lb} and the accompanying proof).

Theorem~\ref{thm:lower-bound-lemma2-form} is stated for the shifted quantity $g(WH)-\Gamma_2$ because the proof begins with an affine pointwise PAL lower bound at each multiplicity level: after averaging, the intercept terms accumulate and their $N_m$-weighted aggregate is exactly $\Gamma_2$, so subtracting $\Gamma_2$ removes an unavoidable constant shift. The tuning constants $\{c_{1,m}\}_{m=1}^M$ (one per multiplicity) play the same role as the ``$c_1$'' in \citet{Zhu2021} or the ``$t_k$'' parameters in \citet{dang2024neural}, generating a family of valid affine bounds that can be optimized for tightness.

\begin{remark}
A key novelty of our analysis is that we compress correlation and imbalance into a single, computable control quantity $A_m$. 
The distribution-sensitive content is concentrated in $A_m$ (our key control quantity). It jointly encodes correlation and imbalance through three effects: (i) $A_m$ increases as $\kappa_m$ decreases, since the $1/\kappa_m$ factor captures spectral near-degeneracy on centered contrast directions, making any spectral lower bound ill-conditioned when $\kappa_m$ is small; (ii) $A_m$ decreases with $N_m$, reflecting the benefit of more multiplicity-$m$ samples; and (iii) within-group long-tail rarity enters via the worst-set term $\max_{S\in\mathcal S_m}\sum_{j\in S}1/N_m^j$, which amplifies rare labels at the combinatorial level, paralleling the frequency-driven scaling effects in imbalanced multi-class NC analyses (e.g., \citet{dang2024neural}).

\end{remark}

The remaining factor $\sqrt{\lambda_W/\lambda_H}\,\rho$ (with $\rho=\|W\|_F^2$) is an $\ell_2$-induced energy scale, closed via the critical-point identity $\|H\|_F^2=(\lambda_W/\lambda_H)\|W\|_F^2$ (Lemma~\ref{lem:a1}), which also explains why $A_m$ is defined to avoid additional dependence on $\rho$. Overall, the bound should be read as a distribution-aware attainability guarantee: after removing $\Gamma_2$, the worst-case negative deviation is governed primarily by $(\kappa_m,\{N_m^j\})$ through $A_m$, so small $\kappa_m$ or extreme within-multiplicity rarity can substantially weaken provable guarantees and destabilize terminal geometry. A more detailed discussion is deferred to Appendix~\ref{app:interpretation:thm1}.

\begin{theorem}[Structural description and geometric structure]
\label{thm:structure-geometry}
Let $(W,H)$ be a global minimizer of $f(W,H)$. Then there exists a constant $C_1>0$ and,
for each $m\in\{1,\dots,M\}$, there exists a constant $C_m>0$, independent of the label set \(S\), such that:

\paragraph{(1) Classifier centering.}
\[
\sum_{k=1}^K w_k=0,
\qquad
\Pi W=W.
\]

\paragraph{(2) Within-group collapse.}
For any $m\in\{1,\dots,M\}$ and any $S\in\mathcal{S}_m$,
\[
h_{m,S,i}=h_{m,S},
\qquad
i=1,\dots,r_{m,S}.
\]

\paragraph{(3) Multiplicity-one centered self-duality.}
Define
\[
\widehat{h}_{1,k}
:=
h_{1,k}
-
\frac{1}{K}\sum_{\ell=1}^K h_{1,\ell},
\qquad
k\in [K].
\]

Then there exists a constant $C_1>0$ such that
\[
\widehat{h}_{1,k}=C_1 w_k,
\qquad
\forall k\in [K].
\]

Equivalently, if we define
\[
u_k:=\frac{w_k}{\sqrt{r_{1,k}}},
\]
then
\[
\widehat{h}_{1,k}
=
C_1\sqrt{r_{1,k}}\,u_k,
\qquad
\forall k\in [K].
\]

\paragraph{(4) ``Label-set generation'' for $m\ge 2$ (controlled by multiplicity-one).}
For any $m\in\{2,\dots,M\}$ and any $S\in\mathcal{S}_m$,
\[
h_{m,S}=C_m\sum_{k\in S} \sqrt{r_{1,k}} u_k.
\]
\end{theorem}

This theorem is the formal statement of Lemma~\ref{lem3}; its detailed proof can be found in Appendix~\ref{appendix1} (see Lemma~\ref{lem3} and the accompanying proof).

Theorem~\ref{thm:structure-geometry} should be read structurally: statements (1)--(4) describe different facets of a single terminal regime, rather than four unrelated properties. The terminal geometry is first fixed on the centered discriminative subspace, then collapses at the label-set granularity, and finally shows that every higher-multiplicity prototype is generated from the multiplicity-one directions through a count-weighted synthesis law.

Statement (1) implies that the discriminative geometry effectively lives in the centered subspace: the $\mathbf{1}$-direction only induces a global logit shift and is therefore a redundant degree of freedom. Once centered, the relevant geometric relations concern genuine inter-class contrast directions.

Statement (2) identifies the natural granularity of multi-label collapse: variability collapses within each label set $S$ (at fixed multiplicity $m$), so the terminal representation assigns a single prototype per label set. This is the correct level at which terminal geometry should be analyzed.

\begin{remark}
The key imbalance signature already appears at multiplicity one. In the imbalance-aware coordinates
\(
u_k=\frac{w_k}{\sqrt{r_{1,k}}},
\)
\(
\widehat h_{1,k}
=
h_{1,k}
-
\frac{1}{K}\sum_{\ell=1}^K h_{1,\ell},
\)
statement (3) gives
\(
\widehat h_{1,k}=C_1 w_k,
\) equivalently \(
\widehat h_{1,k}=C_1\sqrt{r_{1,k}}\,u_k.
\)
Thus, after centering, each multiplicity-one prototype is aligned with its corresponding classifier direction, while the multiplicity-one count profile enters explicitly through the factor $\sqrt{r_{1,k}}$ in the imbalance-aware coordinates. More precisely, statement (3) identifies the multiplicity-one directions as the basic building blocks, and statement (4) shows that the higher-multiplicity geometry is generated from them through explicit count-dependent weights. We do not interpret this as asserting that the rescaled vectors $\{u_k\}_{k=1}^K$ necessarily form an exact centered ETF under general multiplicity-one imbalance.
\end{remark}

\begin{remark}
For $m\ge 2$, statement (4) gives the distinctive synthesis law
\(
h_{m,S}=C_m\sum_{k\in S}\sqrt{r_{1,k}}\,u_k,
\)
which is sharper than naive averaging: every higher-multiplicity prototype is a linear combination of the multiplicity-one directions, with weights determined solely by the multiplicity-one count profile $\{\sqrt{r_{1,k}}\}$. This directly addresses the multiplicity-one imbalance conjecture in \citet{li2023neural} by making the ``scaled-average'' dependence explicit via the above $\sqrt{r_{1,k}}$-weighted synthesis. In the balanced special case where $r_{m,S}\equiv n_m$ within each $\mathcal S_m$ and $r_{1,k}$ is constant across $k$, the weights $\sqrt{r_{1,k}}$ are constant and the generation law reduces to an unweighted tag-wise summation (and, after normalization, to the scaled-average structure in \citet{li2023neural}). A more detailed discussion is deferred to Appendix~\ref{app:interpretation:thm2}.
\end{remark}




\section{Experiments}

Our experiments address two questions: (i) on synthetic multi-label benchmarks, whether late-training geometric collapse is stable and trackable by a small set of scalar metrics; and (ii) when imbalance is introduced \emph{only} within the multiplicity-one subset, whether the late-training geometry deviates from the balanced baseline in a monotone and persistent manner across datasets.

\subsection{Data construction and training setup}

\paragraph{Synthetic multi-label MNIST and CIFAR10.}
We construct multi-label samples using a pad-stack rule: each single-label image is zero-padded to double its original spatial size, and two padded images from different classes are superimposed to form a multiplicity-two sample. Throughout, we consider the case where the maximum multiplicity is $2$. For the training set, we sample $3100$ multiplicity-one images per class and generate $200$ multiplicity-two images per class pair, resulting in $40000$ training samples in total.

\paragraph{Imbalance confined to multiplicity one.}
To isolate the effect of multiplicity-one imbalance, we keep all multiplicity-two samples unchanged and downsample only the multiplicity-one subset.
For ratios $r\in\{0.2,0.1\}$, we keep classes $0$--$4$ intact (each with $3100$ multiplicity-one samples) and downsample classes $5$--$9$ to $r$ of their original counts (i.e., $620$ and $310$ per class for $r=0.2$ and $r=0.1$, respectively).
Consequently, imbalance enters the training data only through the multiplicity-one class counts $\{r_{1,k}\}_{k=1}^K$.

\paragraph{Optimization.}
We train ResNet18 with SGD (batch size $128$, momentum $0.9$) for $200$ epochs, using a cosine-style learning-rate schedule from $0.1$ to $0.001$, with identical hyperparameters for MNIST and CIFAR10.

\subsection{Metrics}

We track four scalar metrics throughout training (y-axes labeled as \textsc{NC1}, \textsc{NC2}, \textsc{NC3}, and \textsc{Angle}):
\begin{itemize}
  \item \textsc{NC1}: within-class collapse (smaller indicates stronger collapse).
  \item \textsc{NC2}: a classifier-related geometric consistency metric.
  \item \textsc{NC3}: classifier--class-mean alignment (self-duality).
  \item \textsc{Angle}: multiplicity-two additivity from feature means (\texttt{angle\_metric}).
\end{itemize}

\begin{figure*}[t]
  \vskip 0.1in
  \begin{center}
    \begin{minipage}{0.24\textwidth}
      \centering
      \includegraphics[width=\linewidth]{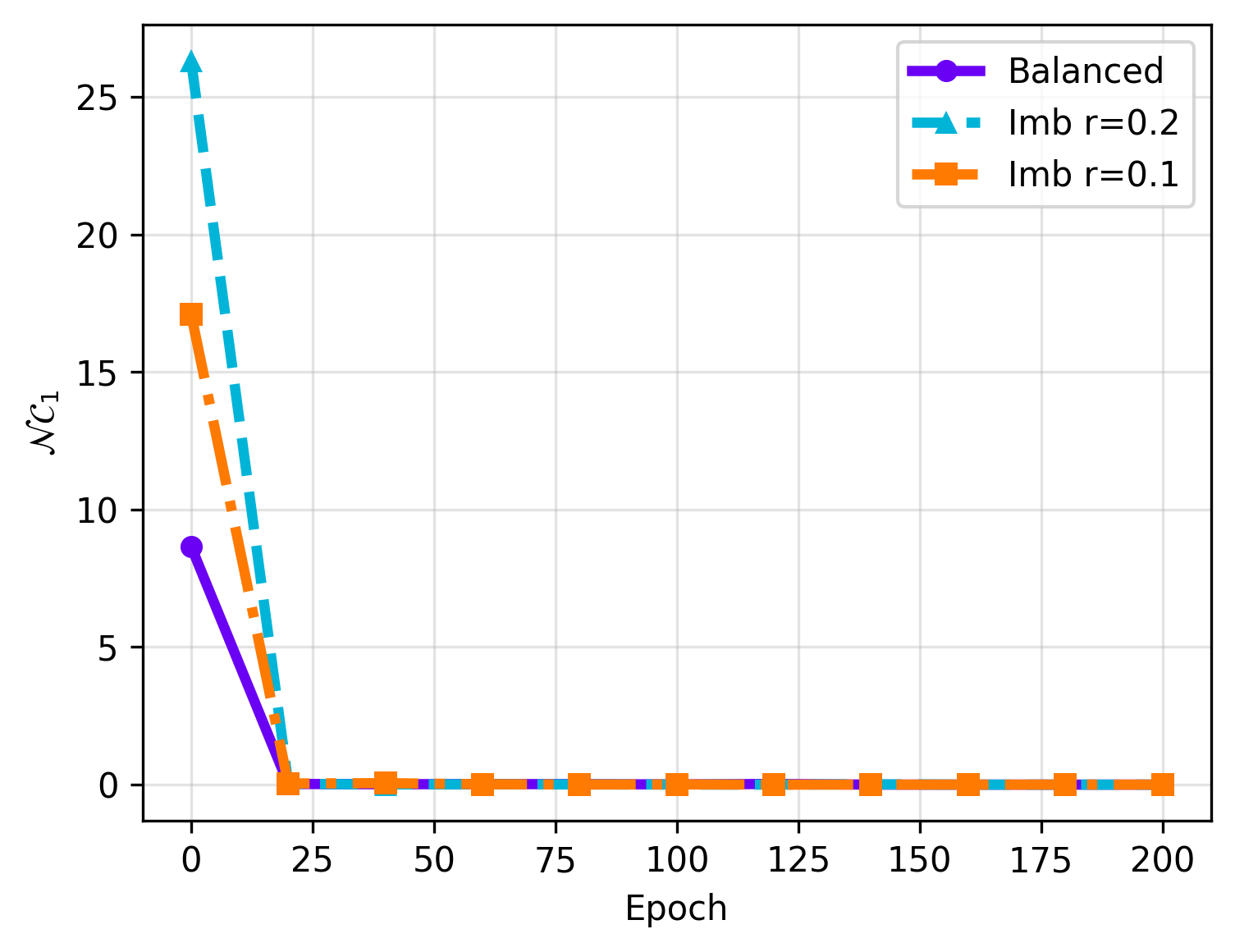}
      \vspace{-0.05in}
      \caption*{\small (a) $\mathcal{NC}_1$ (MLab-MNIST)}
    \end{minipage}
    \hfill
    \begin{minipage}{0.24\textwidth}
      \centering
      \includegraphics[width=\linewidth]{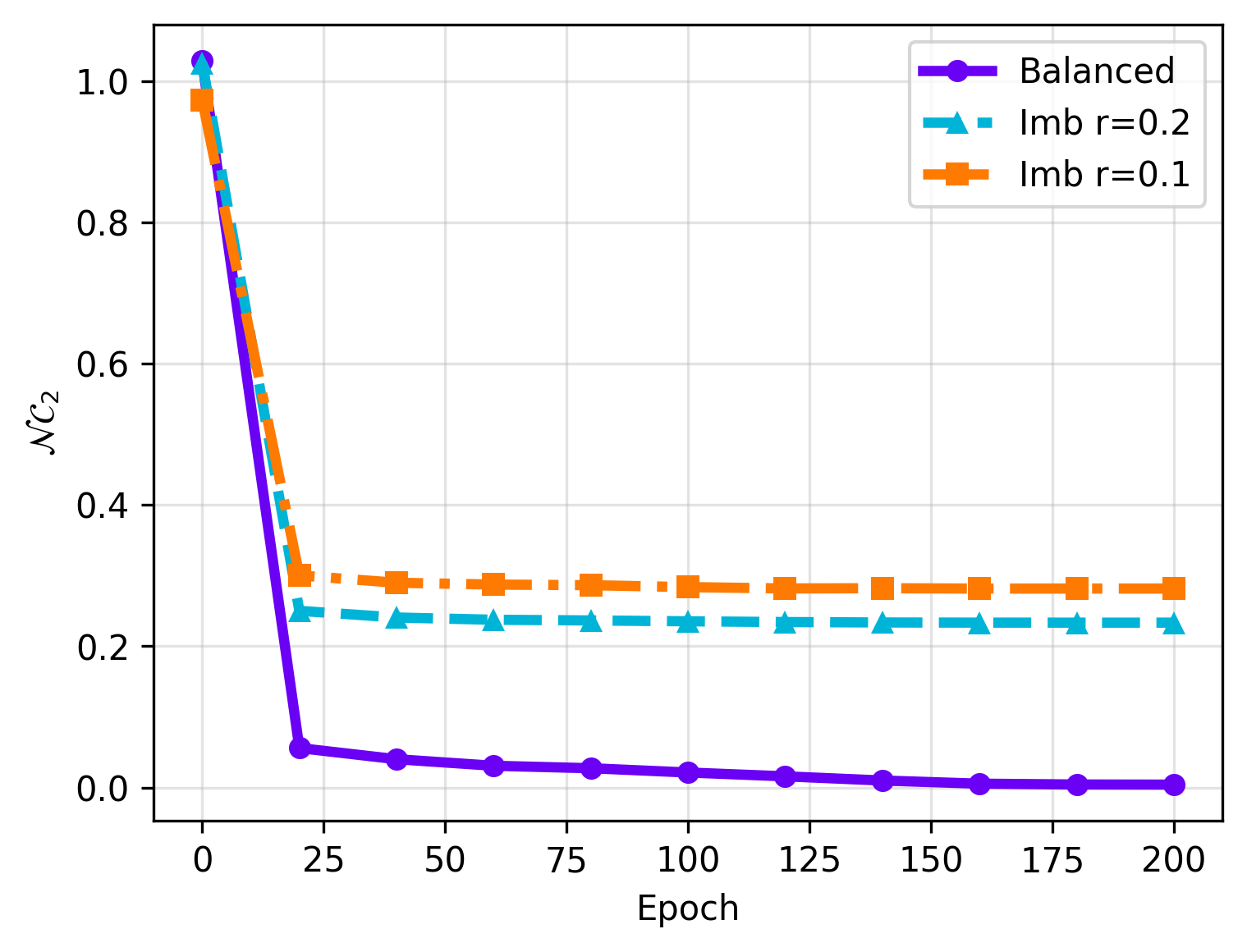}
      \vspace{-0.05in}
      \caption*{\small (b) $\mathcal{NC}_2$ (MLab-MNIST)}
    \end{minipage}
    \hfill
    \begin{minipage}{0.24\textwidth}
      \centering
      \includegraphics[width=\linewidth]{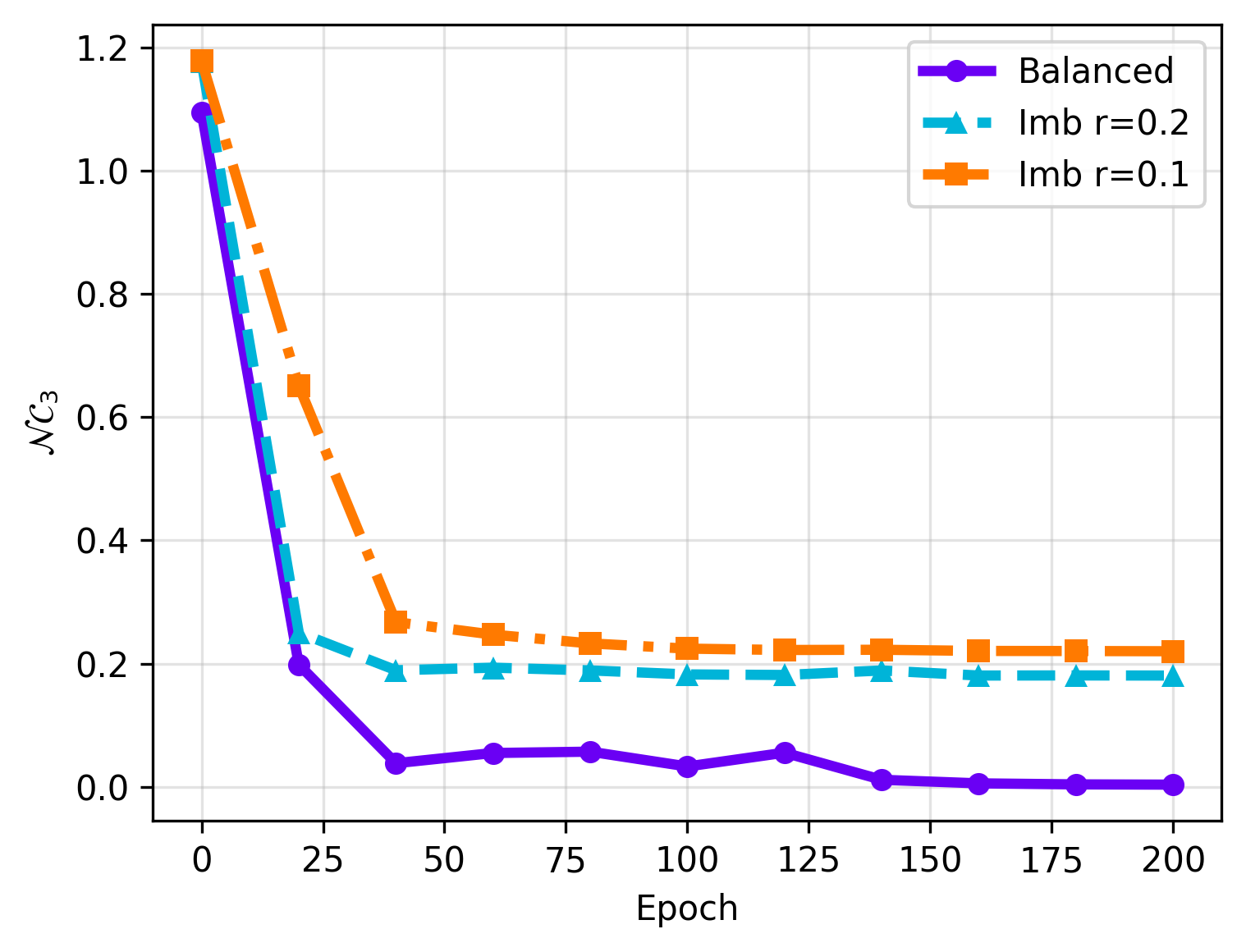}
      \vspace{-0.05in}
      \caption*{\small (c) $\mathcal{NC}_3$ (MLab-MNIST)}
    \end{minipage}
    \hfill
    \begin{minipage}{0.24\textwidth}
      \centering
      \includegraphics[width=\linewidth]{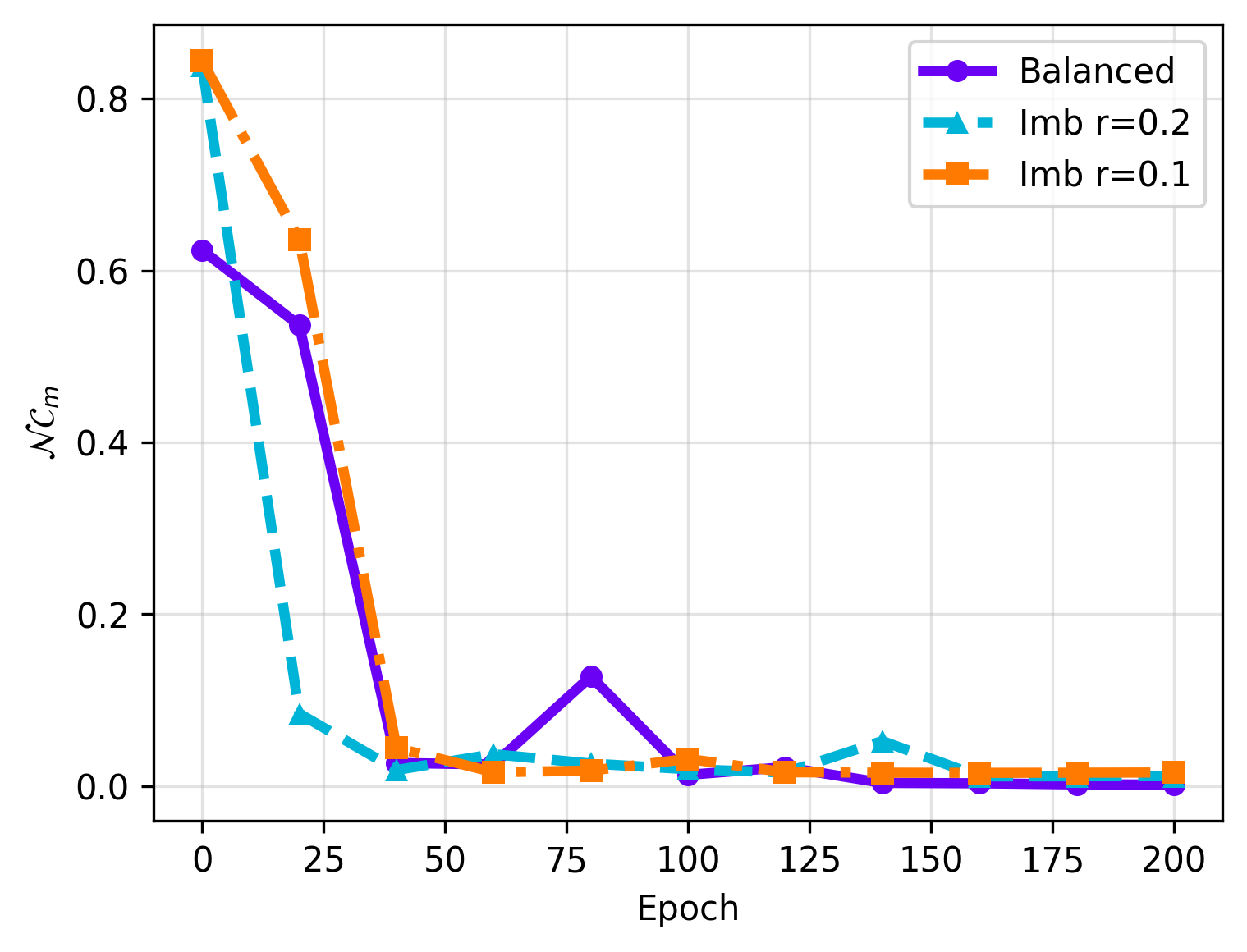}
      \vspace{-0.05in}
      \caption*{\small (d) $\mathcal{NC}_m$ (MLab-MNIST)}
    \end{minipage}

    \vspace{0.1in}

    \begin{minipage}{0.24\textwidth}
      \centering
      \includegraphics[width=\linewidth]{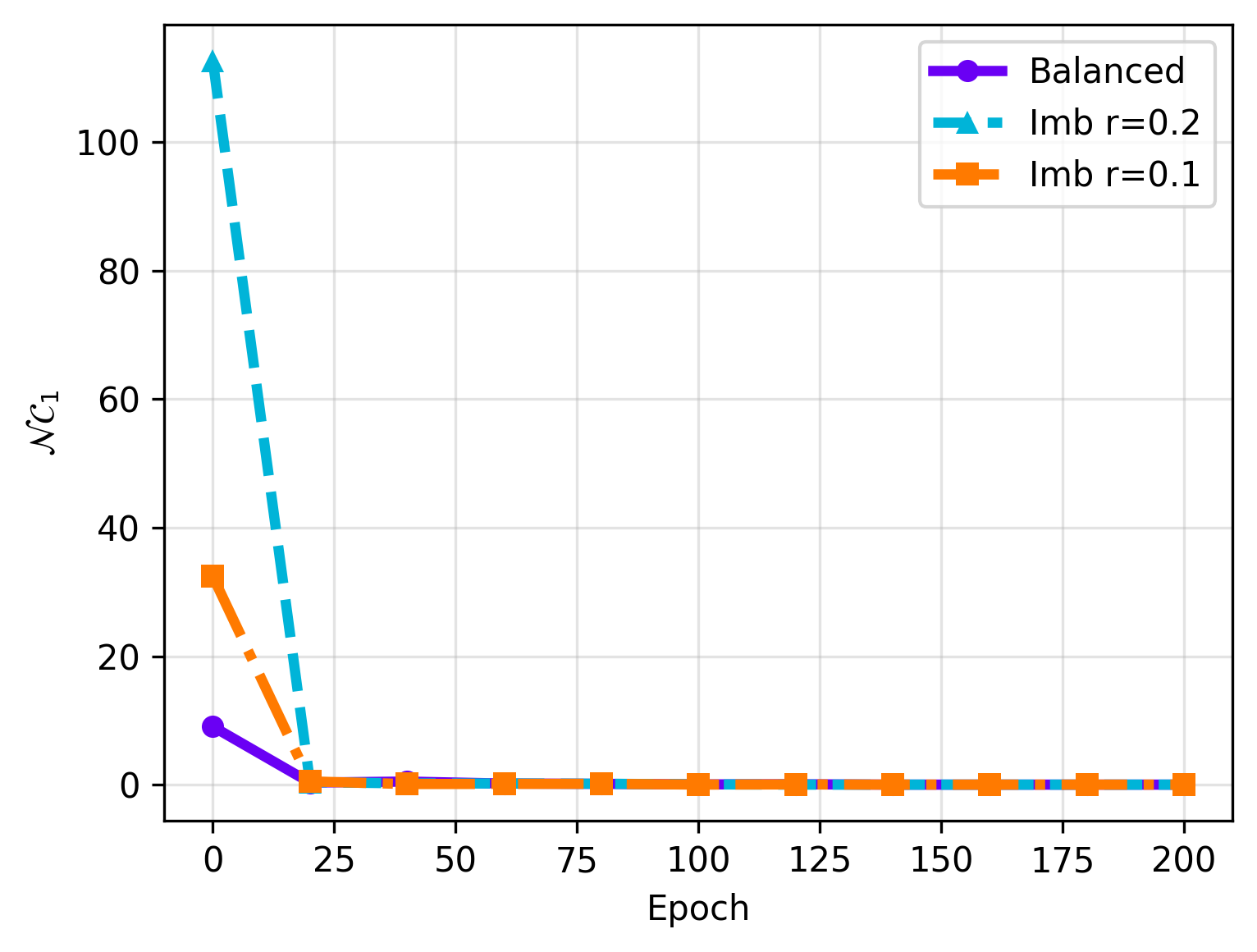}
      \vspace{-0.05in}
      \caption*{\small (e) $\mathcal{NC}_1$ (MLab-CIFAR10)}
    \end{minipage}
    \hfill
    \begin{minipage}{0.24\textwidth}
      \centering
      \includegraphics[width=\linewidth]{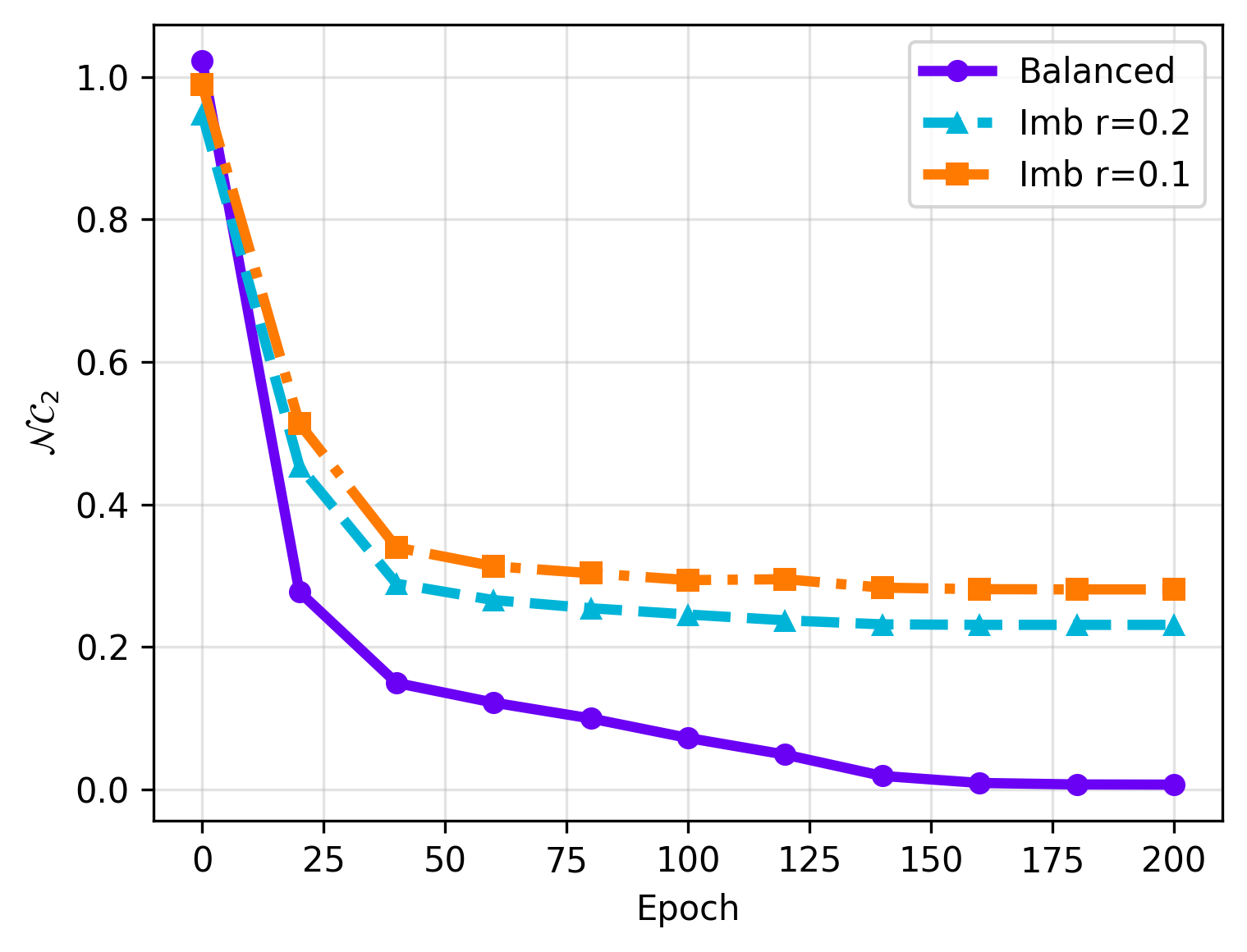}
      \vspace{-0.05in}
      \caption*{\small (f) $\mathcal{NC}_2$ (MLab-CIFAR10)}
    \end{minipage}
    \hfill
    \begin{minipage}{0.24\textwidth}
      \centering
      \includegraphics[width=\linewidth]{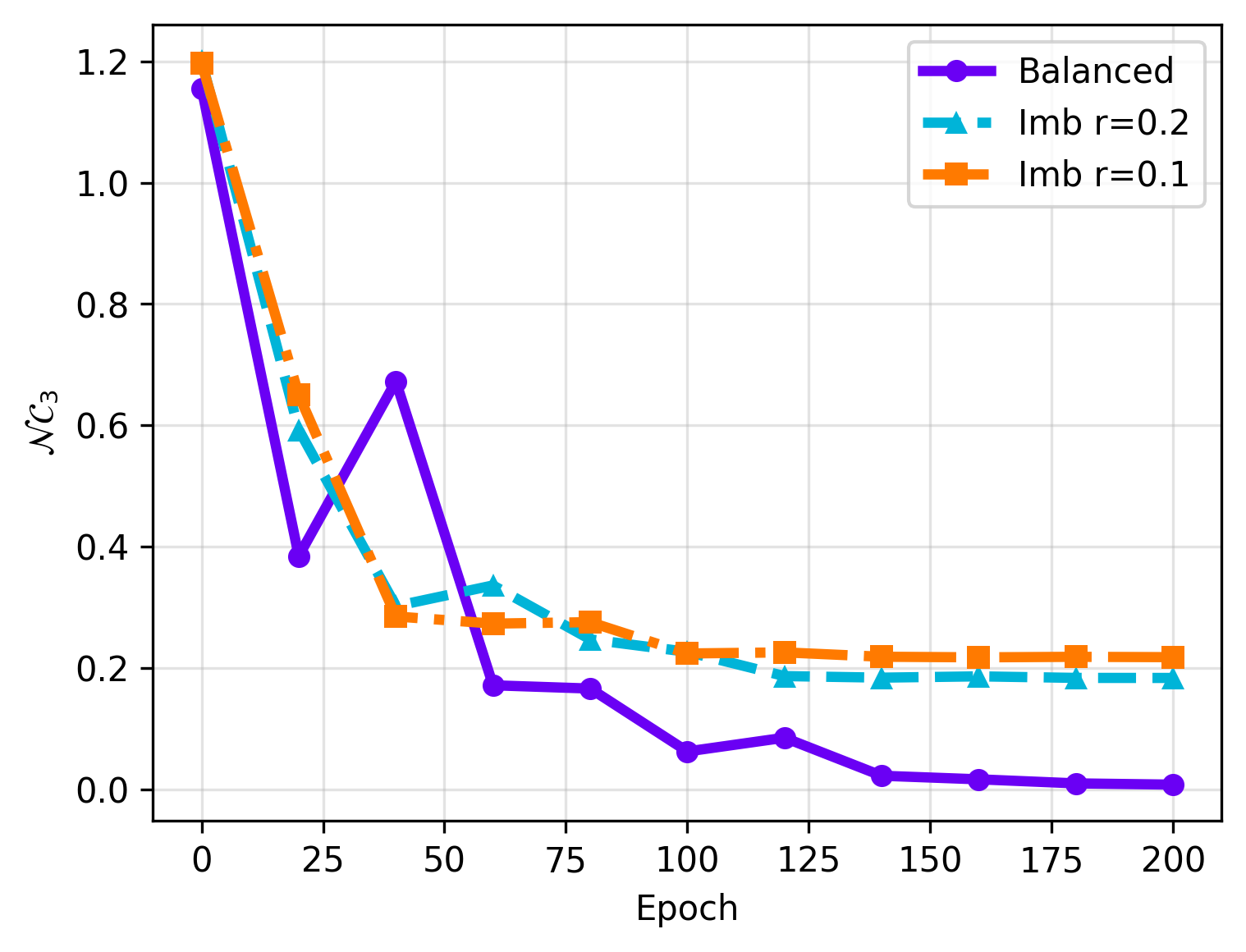}
      \vspace{-0.05in}
      \caption*{\small (g) $\mathcal{NC}_3$ (MLab-CIFAR10)}
    \end{minipage}
    \hfill
    \begin{minipage}{0.24\textwidth}
      \centering
      \includegraphics[width=\linewidth]{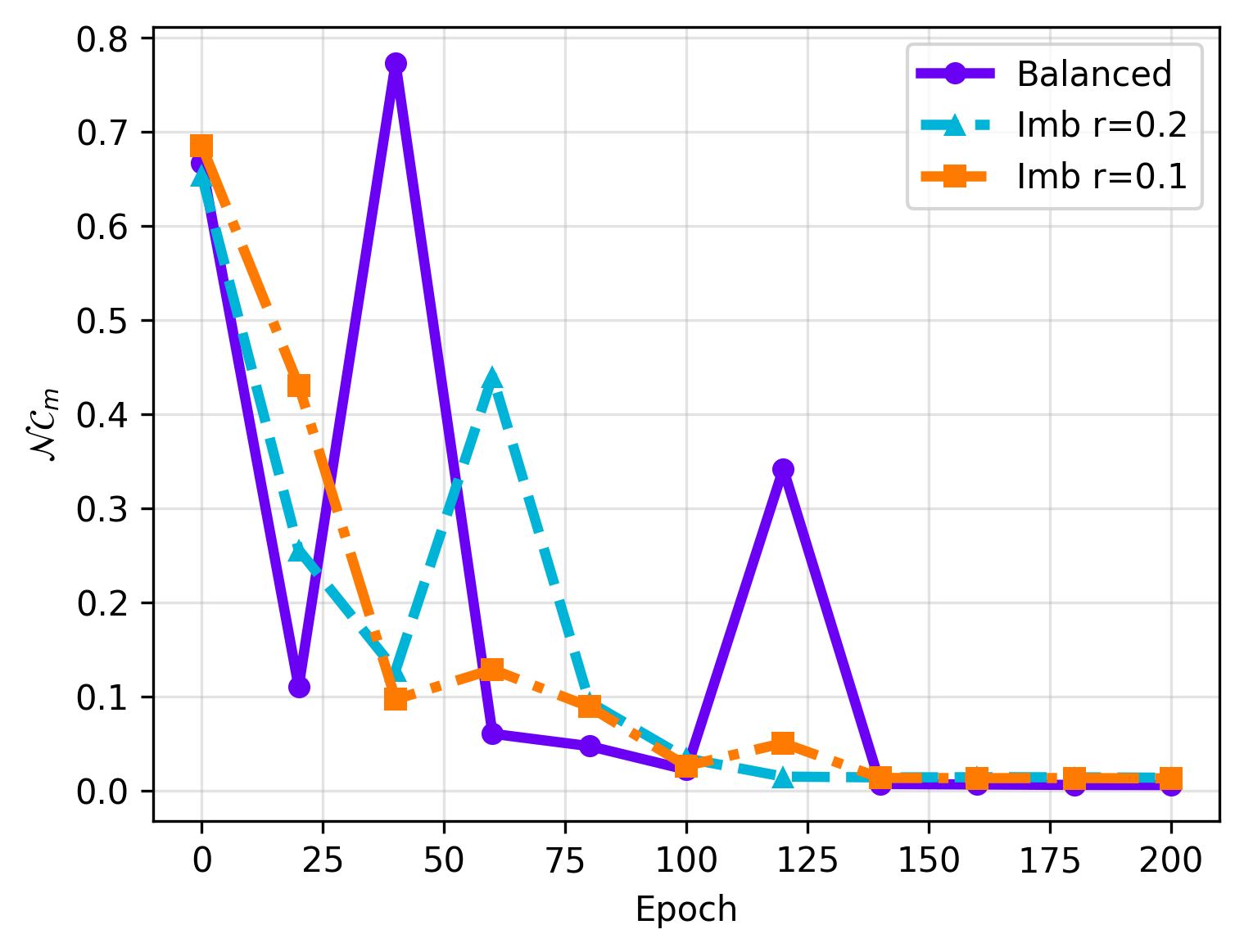}
      \vspace{-0.05in}
      \caption*{\small (h) $\mathcal{NC}_m$ (MLab-CIFAR10)}
    \end{minipage}

    \vspace{0.05in}
    
    \caption{
    \textbf{Training dynamics of geometric-collapse metrics on MLab-MNIST (top) and MLab-CIFAR10 (bottom).}
    Curves compare the balanced baseline with multiplicity-one imbalance ratios $r=0.2$ and $r=0.1$ (only multiplicity-one samples are downsampled, while multiplicity-two samples are kept fixed). 
    (a) $\mathcal{NC}_1$ measures within-class collapse. 
    (b) $\mathcal{NC}_2$ measures a classifier-related geometric consistency metric (in our implementation, it corresponds to the weight-side version).
    (c) $\mathcal{NC}_3$ measures classifier--class-mean alignment (self-duality).
    (d) $\mathcal{NC}_m$ is an angle-based metric for multiplicity-two additivity (smaller indicates better agreement between multiplicity-two means and the sum of their multiplicity-one components).
    }
    \label{fig:mnist_cifar10_4panel}
  \end{center}
  \vskip -0.1in
\end{figure*}

\begin{figure}[t]
  \centering
  \begin{minipage}{0.48\columnwidth}
    \centering
    \includegraphics[width=\linewidth]{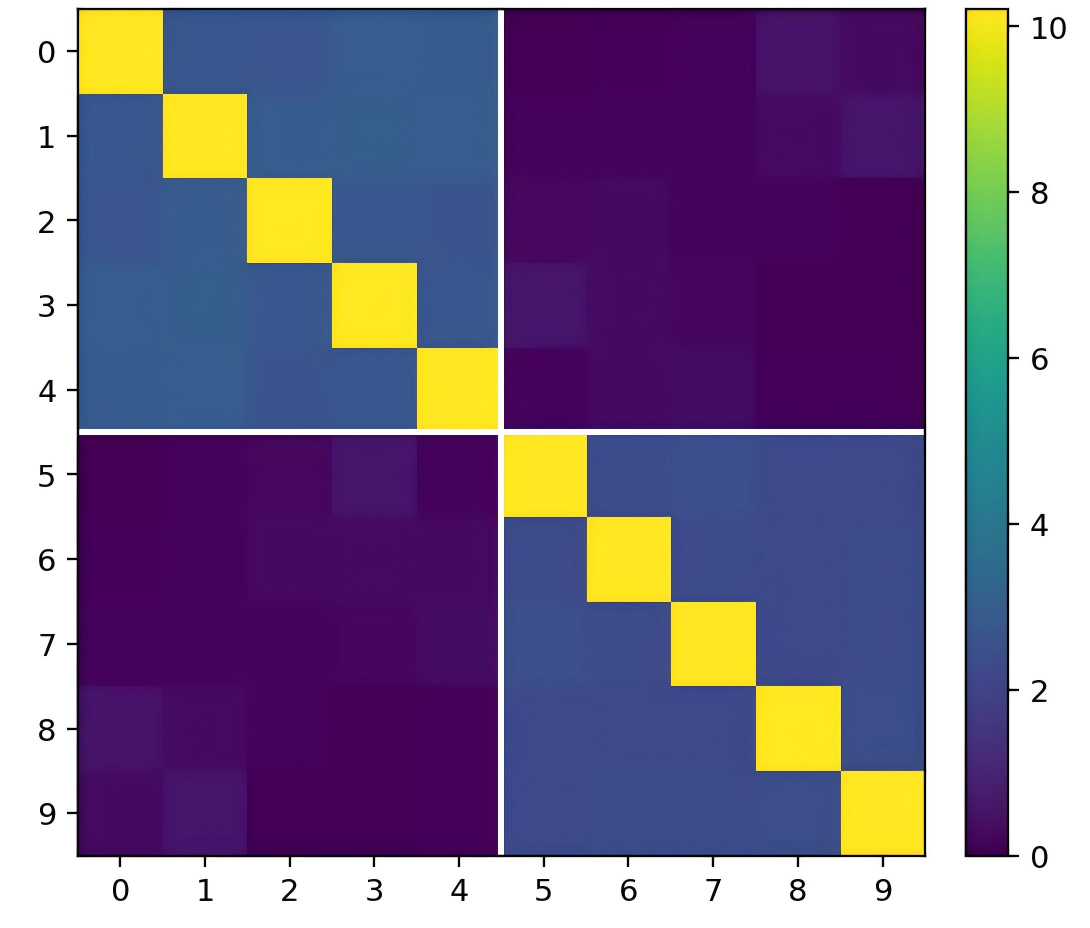}
    \vspace{-6pt}
    {\scriptsize\rmfamily
    \scalebox{0.92}{$
      \lvert\,\mathrm{W}\mathrm{W}^{\mathsf{T}}-\mathrm{c}^{\star}(\mathrm{W})\,\Pi\,\rvert
    $}}
  \end{minipage}
  \hfill
  \begin{minipage}{0.48\columnwidth}
    \centering
    \includegraphics[width=\linewidth]{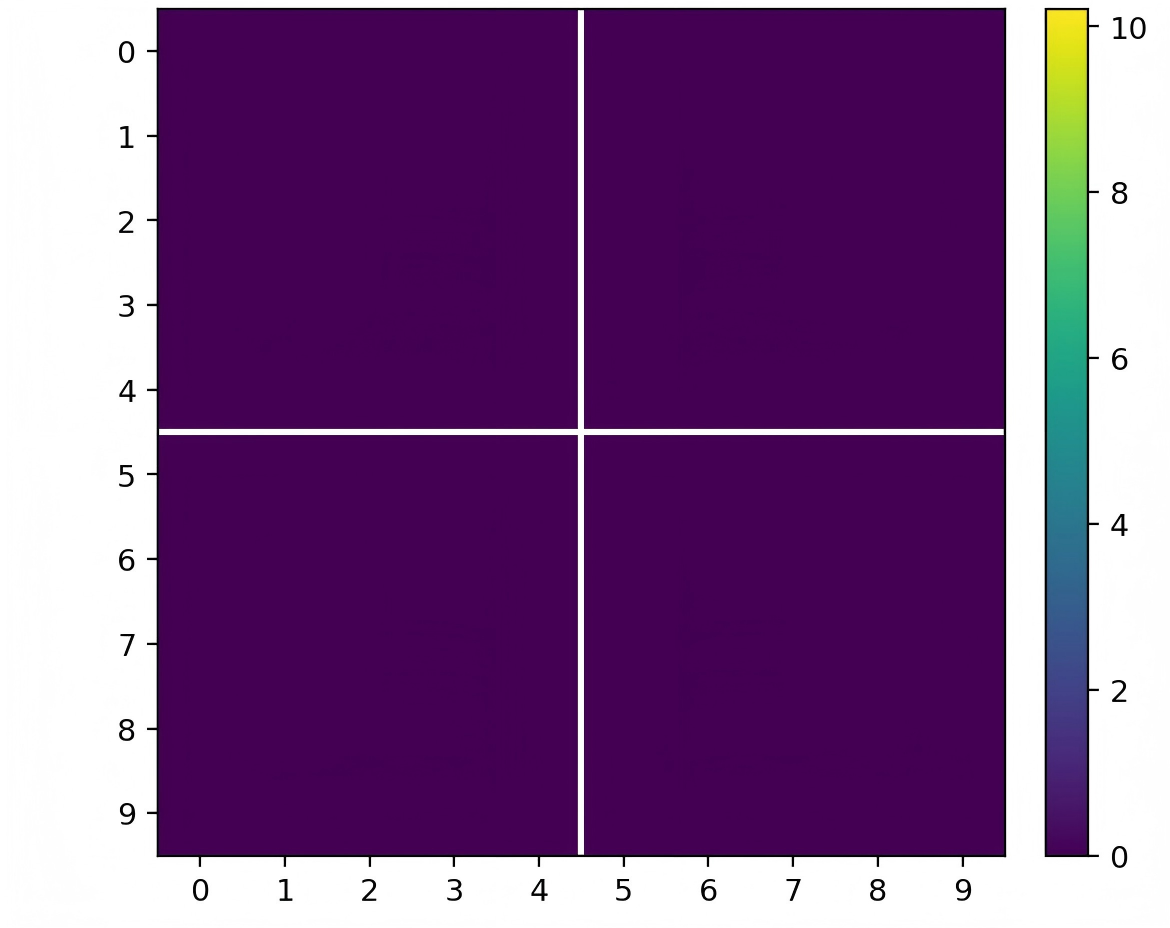}
    \vspace{-6pt}
    {\scriptsize\rmfamily
    \scalebox{0.92}{$
      \lvert\,\mathrm{D}^{-\frac{1}{2}}\mathrm{W}\mathrm{W}^{\mathsf{T}}\mathrm{D}^{-\frac{1}{2}}
      -\mathrm{c}^{\star}(\mathrm{D}^{-\frac{1}{2}}\mathrm{W})\,\Pi\,\rvert
    $}}
  \end{minipage}
\vspace{2pt}
\caption{
  \textbf{Scaling-sensitive Gram residuals under multiplicity-one imbalance.}
  Let $W\in\mathbb{R}^{K\times d}$ be the last-layer classifier weights ($K=10$) and $D=\mathrm{diag}(r_{1,1},\dots,r_{1,K})$ be the multiplicity-one class-count matrix.
  We visualize residuals of the Gram matrices after optimal scale alignment to the ETF-inspired centered template \(c\Pi\).
  The two panels plot absolute residuals (small scale) for $W$ and the theory-motivated scaling $D^{-\frac{1}{2}}W$.
  }
  \label{fig:scaling_residuals}
\end{figure}

\begin{figure}[t]
  \centering
  \includegraphics[width=\columnwidth]{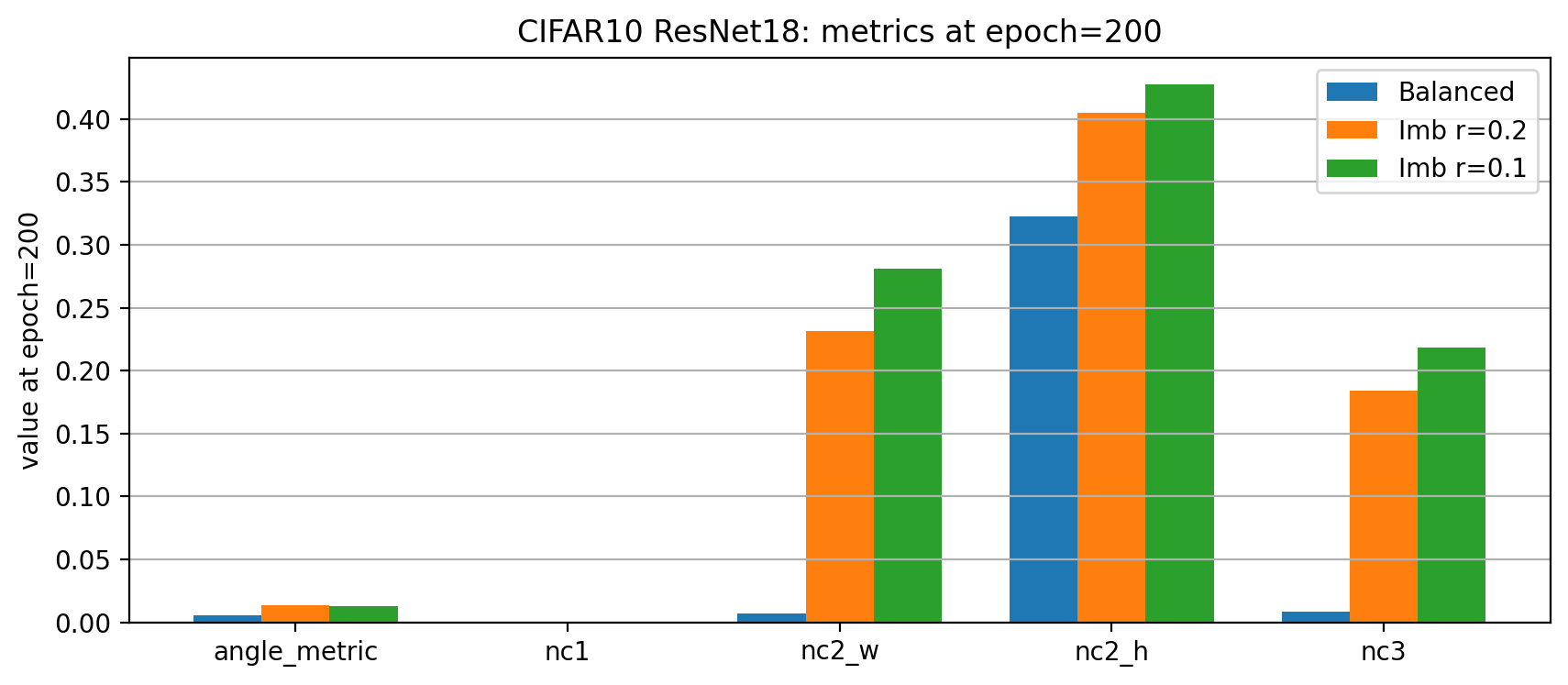}
  \caption{\textbf{Comparison of metrics for CIFAR10 (ResNet18) at epoch 200. }
  The bar chart compares three settings: Balanced, $r=0.2$, and $r=0.1$ (only the multiplicity-one subset is downsampled, while multiplicity-two samples remain unchanged). 
  Metrics shown include angle\_metric, NC1, NC2-W, NC2-H, and NC3.}
  \label{fig:c10_epoch200_bar}
\end{figure}

\begin{table}[t]
  \centering
  \caption{Scaling-sensitive Gram alignment. Smaller residual indicates closer agreement with the centered template.}
  \label{tab:scaling_check}
  \resizebox{\columnwidth}{!}{%
  \begin{tabular}{lcc}
    \toprule
    Scaling on $W$ & Best scale $c^{\star}$ & Residual $\lVert G(A)-c^{\star}\Pi\rVert_{F}$ \\
    \midrule
    $W$ & 41.63 & 36.14 \\
    $D^{-1/2}W\ (= W/\sqrt{r_{1}})$ & 0.0558 & 0.116 \\
    $D^{1/2}W\ (= \sqrt{r_{1}}\,W)$ & $8.52\times 10^{4}$ & $2.33\times 10^{5}$ \\
    \bottomrule
  \end{tabular}}
\end{table}

\subsection{Results}

\paragraph{Late-training geometry under multiplicity-one imbalance.}
Figure~\ref{fig:mnist_cifar10_4panel} compares the balanced baseline with multiplicity-one imbalance ratios $r=0.2$ and $r=0.1$ on MLab-MNIST (top) and MLab-CIFAR10 (bottom).
In the balanced setting, all metrics enter a stable late-training regime and approach small values.
When imbalance is confined to multiplicity one, the terminal geometry deviates systematically: \textsc{NC2} and \textsc{NC3} remain noticeably larger than the balanced baseline, and the deviation increases as $r$ decreases.
The \textsc{Angle} metric also shifts with $r$, indicating that multiplicity-two additivity is indirectly affected through the distorted multiplicity-one geometry; overall, imbalance encoded solely by $\{r_{1,k}\}$ leaves persistent and measurable signatures in the terminal-phase geometry.

\paragraph{Scaling-sensitive residuals.}
Let $W\in\mathbb{R}^{K\times d}$ denote the last-layer classifier weights (rows $w_k^\top$), and let $r_{1,k}$ be the multiplicity-one sample count of class $k$.
Define $D=\mathrm{diag}(r_{1,1},\dots,r_{1,K})$ and $\Pi=I-\frac{1}{K}\mathbf{1}\mathbf{1}^{\top}$.
For any matrix $A\in\mathbb{R}^{K\times d}$, consider the Gram matrix $G(A)=AA^{\top}$ and align it to the centered template $c\Pi$ via
\(
c^{\star}(A)=\arg\min_{c\in\mathbb{R}}\ \|\,G(A)-c\Pi\,\|_{F},
\)
\(
R(A)=G(A)-c^{\star}(A)\Pi.
\)
Figure~\ref{fig:scaling_residuals} compares the absolute residuals $|R(W)|$ and $|R(D^{-1/2}W)|$.
The residual becomes substantially smaller after applying the class-count scaling $D^{-1/2}$, indicating that multiplicity-one imbalance leaves a structured signature in $WW^{\top}$ that is largely explained by the $\sqrt{r_{1,k}}$ scaling.
This provides scaling-sensitive empirical evidence consistent with the count-dependent structure in Theorem~\ref{thm:structure-geometry}: imbalance enters the terminal geometry through the multiplicity-one count profile $\{r_{1,k}\}$, and the corresponding scaled variables yield a Gram structure closer to the centered template.

Figure~\ref{fig:c10_epoch200_bar} summarizes the terminal-phase behavior (epoch 200) on CIFAR10. Compared to the balanced baseline, restricting imbalance to the multiplicity-one subset leads to a clear increase in geometry-related metrics (e.g., NC2 and NC3), and the effect strengthens as $r$ decreases.

To connect this observation to the scaling structure predicted by our analysis, Table~\ref{tab:scaling_check} reports a scaling-sensitive Gram alignment test. For $A\in\{W, D^{-1/2}W, D^{1/2}W\}$, we consider $G(A)=AA^{\top}$ and align it to the centered template $c\Pi$ via
\[
c^{\star}(A)=\arg\min_{c\in\mathbb{R}}\|G(A)-c\Pi\|_{F}.
\]
The resulting residual $\|G(A)-c^{\star}(A)\Pi\|_{F}$ is minimized by the class-count scaling $D^{-1/2}W$, while the opposite scaling $D^{1/2}W$ amplifies the deviation, providing quantitative evidence that multiplicity-one imbalance enters the terminal geometry through the $\sqrt{r_{1,k}}$ scaling.

\section{Conclusion and Future Perspectives}

\subsection{Summary of Key Contributions}
This work advances the geometric theory of deep neural networks for multi-label classification by providing a unified account of Neural Collapse (NC) under correlation and imbalance. Our main contributions are:
(1) Resolving a core gap at the intersection of multi-label NC and imbalance.
We address the open issue emphasized in prior multi-label NC work \cite{li2023neural} by giving, to our knowledge, the first systematic structural characterization of the terminal phase under \emph{general, correlated, and imbalanced} label distributions, thereby extending NC theory toward practical multi-label regimes.
(2) A spectral framework with a scalar organizing invariant.
We introduce a rigorous spectral-control framework and the label covariance spectrum \( \kappa_m \) (from the label second-moment matrix) as a key scalar that controls the distribution-dependent lower-bound geometry.
(3) A generalized spectral-control principle.
We prove that terminal multi-label geometry is not governed by uniform tag-wise averaging in general; instead, the centered label covariance spectrum controls the stability of the geometry by quantifying the weakest centered inter-class contrast directions. The classical ``tag-wise average'' structure \cite{li2023neural} appears only as a special case under suitable label orthogonality and symmetry conditions.
(4) Connecting disparate NC phenomena.
Our framework recovers the balanced multi-label geometry of \cite{li2023neural} as a special case and is consistent with distorted geometries known in imbalanced multi-class settings \cite{Fang2021, Dang2023}, while explaining how multiplicity-one imbalance propagates to higher-multiplicity prototypes through explicit count-dependent weights.
(5) Empirical validation and practical insights.
Synthetic experiments validate our theoretical predictions and provide useful guidance for multi-label system design, including architecture choice, regularization, and augmentation.

\subsection{Theoretical and Practical Implications}
The spectral perspective identifies \textbf{spectral control} as an organizing principle behind NC in multi-label learning, bridging representation learning with classical spectral methods. Practically, \( \kappa_m \) serves as a quantitative diagnostic for dataset complexity and expected stability of terminal geometry, with implications for architecture design, regularization, and robust learning under imbalance.

\subsection{Limitations and Future Work}
Important directions include: the \emph{dimensional bottleneck} regime where feature dimension is below the effective rank, which may exacerbate minority-collapse effects \cite{Fang2021, Dang2023}; dynamic analyses explaining how the spectral-control quantities emerge during training; applications to large-scale real-world multi-label benchmarks such as COCO \cite{Lin2014}; and algorithmic developments (losses/regularizers/architectures) that explicitly leverage label-spectrum structure. A full angle-level characterization of multiplicity-one minority collapse under general imbalance remains an important direction for future work.

\subsection{Concluding Remarks}
By analyzing terminal feature geometry under general label correlation and imbalance, our spectral-control framework strengthens the theoretical foundations of NC and provides practical tools for multi-label learning in real-world settings.

\section*{Impact Statement}

This paper presents work whose goal is to advance the field of Machine
Learning. There are many potential societal consequences of our work, none
which we feel must be specifically highlighted here.

\nocite{langley00}

\bibliographystyle{icml2026}
\bibliography{NCPAL}

@inproceedings{langley00,
 author    = {P. Langley},
 title     = {Crafting Papers on Machine Learning},
 year      = {2000},
 pages     = {1207--1216},
 editor    = {Pat Langley},
 booktitle     = {Proceedings of the 17th International Conference
              on Machine Learning (ICML 2000)},
 address   = {Stanford, CA},
 publisher = {Morgan Kaufmann}
}

@article{Papyan2020,
  title={Prevalence of neural collapse during the terminal phase of deep learning training},
  author={Papyan, Vardan and Han, XY and Donoho, David L},
  journal={Proceedings of the National Academy of Sciences},
  volume={117},
  number={40},
  pages={24652--24663},
  year={2020}
}

@article{Zhu2021,
  title={A geometric analysis of neural collapse with unconstrained features},
  author={Zhu, Zhihui and Ding, Tianyu and Zhou, Jicong and Li, Xin and You, Chong and Sulam, Jeremias and Qu, Qing},
  journal={Advances in Neural Information Processing Systems},
  volume={34},
  pages={29820--29834},
  year={2021}
}

@article{Zhou2022a,
  title={On the optimization landscape of neural collapse under MSE loss: Global optimality with unconstrained features},
  author={Zhou, Jicong and Li, Xin and Ding, Tianyu and You, Chong and Qu, Qing and Zhu, Zhihui},
  journal={International Conference on Machine Learning},
  pages={27179--27202},
  year={2022}
}

@article{Mixon2022,
  title={Neural collapse with unconstrained features},
  author={Mixon, Dustin G and Parshall, Hans and Pi, Jianhao},
  journal={Sampling Theory, Signal Processing, and Data Analysis},
  volume={20},
  number={2},
  year={2022}
}

@article{Galanti2022b,
  title={On the role of neural collapse in transfer learning},
  author={Galanti, Tomer and Gy{\"o}rgy, Andras and Hutter, Marcus},
  journal={International Conference on Learning Representations},
  year={2022}
}

@article{Li2022,
  title={Principled and efficient transfer learning of deep models via neural collapse},
  author={Li, Xin and Liu, Shufei and Zhou, Jicong and Lu, Xiliang and Fernandez-Granda, Carlos and Zhu, Zhihui and Qu, Qing},
  journal={arXiv preprint arXiv:2212.12206},
  year={2022}
}

@article{Ji2022,
  title={An unconstrained layer-peeled perspective on neural collapse},
  author={Ji, Wenlong and Lu, Yaodong and Zhang, Yihao and Deng, Zhiquan and Su, Weijie J},
  journal={International Conference on Learning Representations},
  year={2022}
}

@article{Thrampoulidis2022,
  title={Imbalance trouble: Revisiting neural-collapse geometry},
  author={Thrampoulidis, Christos and Kini, Ganesh Ramachandra and Vakilian, Vahab and Behnia, Tina},
  journal={Advances in Neural Information Processing Systems},
  volume={35},
  pages={27225--27238},
  year={2022}
}

@article{Hong2023,
  title={Neural collapse for unconstrained feature model under cross-entropy loss with imbalanced data},
  author={Hong, Weizhi and Ling, Shi},
  journal={arXiv preprint arXiv:2309.09725},
  year={2023}
}

@article{Behnia2023,
  title={On the implicit geometry of cross-entropy parameterizations for label-imbalanced data},
  author={Behnia, Tina and Kini, Ganesh Ramachandra and Vakilian, Vahab and Thrampoulidis, Christos},
  journal={International Conference on Artificial Intelligence and Statistics},
  pages={10815--10838},
  year={2023}
}

@article{Menon2019,
  title={Multilabel reductions: what is my loss optimising?},
  author={Menon, Aditya Krishna and Rawat, Ankit Singh and Reddi, Sashank and Kumar, Sanjiv},
  journal={Advances in Neural Information Processing Systems},
  volume={32},
  year={2019}
}

@article{Liu2021,
  title={The emerging trends of multi-label learning},
  author={Liu, Wei and Wang, Huan and Shen, Xianfang and Tsang, Ivor W},
  journal={IEEE Transactions on Pattern Analysis and Machine Intelligence},
  volume={44},
  number={11},
  pages={7955--7974},
  year={2021}
}

@article{LeCun2015,
  title={Deep learning},
  author={LeCun, Yann and Bengio, Yoshua and Hinton, Geoffrey},
  journal={Nature},
  volume={521},
  number={7553},
  pages={436--444},
  year={2015}
}

@article{Bengio2013,
  title={Representation learning: A review and new perspectives},
  author={Bengio, Yoshua and Courville, Aaron and Vincent, Pascal},
  journal={IEEE Transactions on Pattern Analysis and Machine Intelligence},
  volume={35},
  number={8},
  pages={1798--1828},
  year={2013}
}

@article{Han2022,
  title={Neural collapse under MSE loss: Proximity to and dynamics on the central path},
  author={Han, XY and Papyan, Vardan and Donoho, David L},
  journal={International Conference on Learning Representations},
  year={2022}
}

@article{dang2024neural,
  title={Neural collapse for cross-entropy class-imbalanced learning with unconstrained relu feature model},
  author={Dang, Hien and Tran, Tho and Nguyen, Tan and Ho, Nhat},
  journal={arXiv preprint arXiv:2401.02058},
  year={2024}
}

@article{Zhou2022b,
  title={Are all losses created equal: A neural collapse perspective},
  author={Zhou, Jicong and You, Chong and Li, Xin and Liu, Kang and Liu, Sheng and Qu, Qing and Zhu, Zhihui},
  journal={Advances in Neural Information Processing Systems},
  volume={35},
  year={2022}
}

@inproceedings{li2023neural,
  title={Neural Collapse in Multi-label Learning with Pick-all-label Loss},
  author={Li, Pengyu and Li, Xiao and Wang, Yutong and Qu, Qing},
  booktitle={International Conference on Machine Learning},
  pages={28060--28094},
  year={2024},
  organization={PMLR}
}

@article{Fang2021,
  title={Exploring deep neural networks via layer-peeled model: Minority collapse in imbalanced training},
  author={Fang, Cong and He, Haoyu and Long, Qingcong and Su, Weijie J},
  journal={Proceedings of the National Academy of Sciences},
  volume={118},
  number={43},
  pages={e2103091118},
  year={2021}
}

@article{Dang2023,
  title={Neural collapse in deep linear network: From balanced to imbalanced data},
  author={Dang, Hien and Nguyen, Thanh and Tran, Tan and Tran, Hien and Ho, Nhat},
  journal={arXiv preprint arXiv:2301.00437},
  year={2023}
}

@article{Lin2014,
  title={Microsoft coco: Common objects in context},
  author={Lin, Tsung-Yi and Maire, Michael and Belongie, Serge and Hays, James and Perona, Pietro and Ramanan, Deva and Dollár, Piotr and Zitnick, C Lawrence},
  journal={European conference on computer vision},
  pages={740--755},
  year={2014},
  publisher={Springer}
}

@inproceedings{tirer2022extended,
  title     = {Extended Unconstrained Features Model for Exploring Deep Neural Collapse},
  author    = {Tirer, Tom and Bruna, Joan},
  booktitle = {International Conference on Machine Learning},
  pages     = {21478--21505},
  year      = {2022},
  month     = {June},
  publisher = {PMLR}
}

@article{sukenik2023deep,
  title     = {Deep Neural Collapse is Provably Optimal for the Deep Unconstrained Features Model},
  author    = {Súkeník, Peter and Mondelli, Marco and Lampert, Christoph H.},
  journal   = {Advances in Neural Information Processing Systems},
  volume    = {36},
  pages     = {52991--53024},
  year      = {2023}
}

@inproceedings{tirer2023perturbation,
  title     = {Perturbation Analysis of Neural Collapse},
  author    = {Tirer, Tom and Huang, Haozheng and Niles-Weed, Jonathan},
  booktitle = {International Conference on Machine Learning},
  pages     = {34301--34329},
  year      = {2023},
  month     = {July},
  publisher = {PMLR}
}

@misc{yang2022inducingneuralcollapseimbalanced,
      title={Inducing Neural Collapse in Imbalanced Learning: Do We Really Need a Learnable Classifier at the End of Deep Neural Network?}, 
      author={Yibo Yang and Shixiang Chen and Xiangtai Li and Liang Xie and Zhouchen Lin and Dacheng Tao},
      year={2022},
      eprint={2203.09081},
      archivePrefix={arXiv},
      primaryClass={cs.LG},
      url={https://arxiv.org/abs/2203.09081}, 
}

@article{dembczynski2012label,
  title={On label dependence and loss minimization in multi-label classification},
  author={Dembczy{\'n}ski, Krzysztof and Waegeman, Willem and Cheng, Weiwei and H{\"u}llermeier, Eyke},
  journal={Machine Learning},
  volume={88},
  number={1},
  pages={5--45},
  year={2012},
  publisher={Springer}
}

@InProceedings{pmlr-v19-gao11a,
  title = 	 {On the Consistency of Multi-Label Learning},
  author = 	 {Gao, Wei and Zhou, Zhi-Hua},
  booktitle = 	 {Proceedings of the 24th Annual Conference on Learning Theory},
  pages = 	 {341--358},
  year = 	 {2011},
  editor = 	 {Kakade, Sham M. and von Luxburg, Ulrike},
  volume = 	 {19},
  series = 	 {Proceedings of Machine Learning Research},
  address = 	 {Budapest, Hungary},
  month = 	 {09--11 Jun},
  publisher =    {PMLR},
  url = 	 {https://proceedings.mlr.press/v19/gao11a.html}
}

@misc{blondel2020learningfenchelyounglosses,
      title={Learning with Fenchel-Young Losses}, 
      author={Mathieu Blondel and André F. T. Martins and Vlad Niculae},
      year={2020},
      eprint={1901.02324},
      archivePrefix={arXiv},
      primaryClass={stat.ML},
      url={https://arxiv.org/abs/1901.02324}, 
}

@InProceedings{pmlr-v35-samei14,
  title = 	 {Sample Compression for Multi-label Concept Classes},
  author = 	 {Samei, Rahim and Semukhin, Pavel and Yang, Boting and Zilles, Sandra},
  booktitle = 	 {Proceedings of The 27th Conference on Learning Theory},
  pages = 	 {371--393},
  year = 	 {2014},
  editor = 	 {Balcan, Maria Florina and Feldman, Vitaly and Szepesvári, Csaba},
  volume = 	 {35},
  series = 	 {Proceedings of Machine Learning Research},
  address = 	 {Barcelona, Spain},
  month = 	 {13--15 Jun},
  publisher =    {PMLR},
  url = 	 {https://proceedings.mlr.press/v35/samei14.html}
}

@article{2014Generalizing,
  title={Generalizing Labeled and Unlabeled Sample Compression to Multi-label Concept Classes},
  author={ Samei, Rahim  and  Yang, Boting  and  Zilles, Sandra },
  journal={Springer, Cham},
  year={2014},
}

@misc{xu2014localrademachercomplexitymultilabel,
      title={Local Rademacher Complexity for Multi-label Learning}, 
      author={Chang Xu and Tongliang Liu and Dacheng Tao and Chao Xu},
      year={2014},
      eprint={1410.6990},
      archivePrefix={arXiv},
      primaryClass={stat.ML},
      url={https://arxiv.org/abs/1410.6990}, 
}

@misc{reeve2020optimisticboundsmultioutputprediction,
      title={Optimistic bounds for multi-output prediction}, 
      author={Henry WJ Reeve and Ata Kaban},
      year={2020},
      eprint={2002.09769},
      archivePrefix={arXiv},
      primaryClass={stat.ML},
      url={https://arxiv.org/abs/2002.09769}, 
}

@inproceedings{2010Bayes,
  title={Bayes optimal multilabel classification via probabilistic classifier chains},
  author={Cheng, Weiwei and H{\"u}llermeier, Eyke and Dembczynski, Krzysztof J},
  booktitle={Proceedings of the 27th international conference on machine learning (ICML-10)},
  pages={279--286},
  year={2010}
}

@misc{chang2020tamingpretrainedtransformersextreme,
      title={Taming Pretrained Transformers for Extreme Multi-label Text Classification}, 
      author={Wei-Cheng Chang and Hsiang-Fu Yu and Kai Zhong and Yiming Yang and Inderjit Dhillon},
      year={2020},
      eprint={1905.02331},
      archivePrefix={arXiv},
      primaryClass={cs.LG},
      url={https://arxiv.org/abs/1905.02331}, 
}

@misc{ridnik2021mldecoderscalableversatileclassification,
      title={ML-Decoder: Scalable and Versatile Classification Head}, 
      author={Tal Ridnik and Gilad Sharir and Avi Ben-Cohen and Emanuel Ben-Baruch and Asaf Noy},
      year={2021},
      eprint={2111.12933},
      archivePrefix={arXiv},
      primaryClass={cs.CV},
      url={https://arxiv.org/abs/2111.12933}, 
}

@misc{yu2020learningdiversediscriminativerepresentations,
      title={Learning Diverse and Discriminative Representations via the Principle of Maximal Coding Rate Reduction}, 
      author={Yaodong Yu and Kwan Ho Ryan Chan and Chong You and Chaobing Song and Yi Ma},
      year={2020},
      eprint={2006.08558},
      archivePrefix={arXiv},
      primaryClass={cs.LG},
      url={https://arxiv.org/abs/2006.08558}, 
}

@misc{chan2021redunetwhiteboxdeepnetwork,
      title={ReduNet: A White-box Deep Network from the Principle of Maximizing Rate Reduction}, 
      author={Kwan Ho Ryan Chan and Yaodong Yu and Chong You and Haozhi Qi and John Wright and Yi Ma},
      year={2021},
      eprint={2105.10446},
      archivePrefix={arXiv},
      primaryClass={cs.LG},
      url={https://arxiv.org/abs/2105.10446}, 
}

@misc{LuSteinerberger2020NeuralCollapseCE,
  title         = {Neural Collapse with Cross-Entropy Loss},
  author        = {Lu, Jianfeng and Steinerberger, Stefan},
  year          = {2020},
  eprint        = {2012.08465},
  archivePrefix = {arXiv},
  primaryClass  = {cs.LG},
  doi           = {10.48550/arXiv.2012.08465},
  url           = {https://arxiv.org/abs/2012.08465}
}

@inproceedings{PilanciErgen2020ConvexRegularizers,
  title     = {Neural Networks are Convex Regularizers: Exact Polynomial-time Convex Optimization Formulations for Two-layer Networks},
  author    = {Pilanci, Mert and Ergen, Tolga},
  booktitle = {Proceedings of the 37th International Conference on Machine Learning},
  series    = {Proceedings of Machine Learning Research},
  volume    = {119},
  pages     = {7695--7705},
  year      = {2020},
  publisher = {PMLR},
  url       = {https://proceedings.mlr.press/v119/pilanci20a.html}
}

@inproceedings{RangamaniBanburskiFahey2022WeightDecay,
  title     = {Neural Collapse in Deep Homogeneous Classifiers and the Role of Weight Decay},
  author    = {Rangamani, Akshay and Banburski-Fahey, Andrzej},
  booktitle = {ICASSP 2022 -- 2022 IEEE International Conference on Acoustics, Speech and Signal Processing (ICASSP)},
  pages     = {4243--4247},
  year      = {2022},
  publisher = {IEEE}
}

@inproceedings{KothapalliTirerBruna2023GNN,
  title     = {A Neural Collapse Perspective on Feature Evolution in Graph Neural Networks},
  author    = {Kothapalli, Vignesh and Tirer, Tom and Bruna, Joan},
  booktitle = {Advances in Neural Information Processing Systems},
  volume    = {36},
  year      = {2023}
}

@inproceedings{Kang2020Decoupling,
  title     = {Decoupling Representation and Classifier for Long-Tailed Recognition},
  author    = {Kang, Bingyi and Xie, Saining and Rohrbach, Marcus and Yan, Zhicheng and Gordo, Albert and Feng, Jiashi and Kalantidis, Yannis},
  booktitle = {International Conference on Learning Representations},
  year      = {2020},
  url       = {https://openreview.net/forum?id=r1gRTCVFvB}
}

@inproceedings{Cao2019LDAM,
  title     = {Learning Imbalanced Datasets with Label-Distribution-Aware Margin Loss},
  author    = {Cao, Kaidi and Wei, Colin and Gaidon, Adrien and Arechiga, Nikos and Ma, Tengyu},
  booktitle = {Advances in Neural Information Processing Systems},
  volume    = {32},
  pages     = {1567--1578},
  year      = {2019},
  url       = {https://arxiv.org/abs/1906.07413}
}

\newpage
\appendix
\onecolumn

\section{Proof sketches}
The two main theorems are proved along two complementary threads.
Theorem~\ref{thm:lower-bound-lemma2-form} provides a distribution-aware attainability bound (a quantitative statement),
whereas Theorem~\ref{thm:structure-geometry} extracts the terminal geometry by tracking the equality conditions (a qualitative statement).
At a high level, the proof first reduces the PAL-CE nonlinearity to an additively separable linear form, then compresses all
distributional difficulty into spectral and counting quantities, and finally recovers the structural terminal relations by enforcing equality
across the entire chain of inequalities.

\smallskip
\noindent\textbf{(I) Proof sketch of Theorem~\ref{thm:lower-bound-lemma2-form} (Appendix: Lemma~\ref{lem:lemma2_lb}).}
Fix arbitrary positive constants $\{c_{1,m}\}_{m=1}^M$.
The starting point is an affine (linear-plus-constant) pointwise lower bound for the PAL loss at each multiplicity level:
for any $S\in\mathcal S_m$ and any logits vector $z\in\mathbb R^K$,
Lemma~\ref{lem:your-lemma} shows
\(
\mathcal L_{\mathrm{PAL}}(z,y_S)
\;\ge\;
\gamma_{1,m}\Bigl\langle \mathbf 1-\frac{K}{m}\mathbf 1_S,\;z\Bigr\rangle
+
c_{2,m},
\)
\(
\gamma_{1,m}=\frac{1}{1+c_{1,m}}\cdot\frac{m}{K-m}.
\)
Averaging over samples and summing over multiplicities yields the shifted objective $g(WH)-\Gamma_2$ with
$\Gamma_2=\frac{1}{N}\sum_{m=1}^M N_m c_{2,m}$.

Next, one rewrites the aggregated linear term at a fixed multiplicity $m$ by exchanging summations and using
$\sum_{j=1}^K \mathbf 1_S(j)=m$.
This collapses the combinatorial $(S,i)$-sum into a $K$-term expression involving group-wise means
(cf.\ the identities leading to Lemma~\ref{lem:S5_theta_H_interface}):
\(
\sum_{k=1}^K \Bigl(\bar h_m-\frac{K}{m}\bar h_m^{\,k}\Bigr)^\top w_k.
\)
Applying Young's inequality with a tunable constant $c_{3,m}>0$ converts each inner product into a sum of quadratic terms,
separating the dependence on $W$ and on the mean-difference vectors.
Collecting these vectors into the matrix $\Theta_m$, where \(
\Theta_m(j,:) := \Bigl(\bar h_m-\frac{K}{m}\bar h_m^{\,j}\Bigr)^{\top},
\)
\(j=1,\dots,K
\) (see Lemma~\ref{lem:S5_theta_H_interface}), yields a term proportional to $\|\Theta_m\|_F^2$.

The spectral route then enters.
We show that $\Theta_m$ is centered in the class direction (equivalently $\Pi\Theta_m=\Theta_m$),
allowing one to work entirely on $\mathrm{range}(\Pi)$.
Then the  matrix-form spectral lower bound implies that the centered label covariance controls the mean-difference energy:
\(
\|\Theta_m\|_F^2
\;\le\;
\frac{1}{\kappa_m}\,
\mathrm{Tr}\!\bigl(\Theta_m^\top (\Pi G_m\Pi)\Theta_m\bigr).
\)
Combining this with the trace/Frobenius reductions used in produces the factor $\frac{m(K-m)}{K}$,
and Lemma~\ref{lem:S5_theta_H_interface} further bounds $\|\Theta_m\|_F^2$ by a data-dependent coefficient times the feature energy
$\sum_{S,i}\|h_{m,S,i}\|_2^2$, where the worst-set term
$\max_{S\in\mathcal S_m}\sum_{j\in S}1/N_m^j$ quantifies within-group rarity amplification.

Finally, one balances the two Young terms by choosing $c_{3,m}$ (producing the square-root structure in $A_m$),
sums over $m$, and closes the bound using the critical-point scaling identity (Lemma~\ref{lem:a1}),
$\|H\|_F^2=(\lambda_W/\lambda_H)\|W\|_F^2$, so that all feature energies are expressed on the single scale
$\rho=\|W\|_F^2$.
This yields the stated lower bound in Theorem~\ref{thm:lower-bound-lemma2-form} .

\smallskip
\noindent\textbf{(II) Proof sketch of Theorem~\ref{thm:structure-geometry} (Appendix: Lemma~\ref{lem3})}
The structural theorem is obtained by enforcing equality throughout the inequality chain underlying Theorem~\ref{thm:lower-bound-lemma2-form}.
Since $(W,H)$ is a global minimizer, it must attain the tightness conditions of the PAL affine bound (Lemma~\ref{lem:your-lemma}) and the auxiliary
inequalities used in Lemma~\ref{lem:lemma2_lb}.

First, the classifier-centering phenomenon follows from eliminating the redundant $\mathbf 1$-direction in logits:
projecting onto $\mathrm{range}(\Pi)$ does not worsen the loss while it does not increase the $\ell_2$ penalty,
so any global minimizer may be taken to satisfy $\Pi W=W$, equivalently $\sum_{k=1}^K w_k=0$
(see the centering part of Lemma~\ref{lem3}).

Second, within-group collapse is obtained by symmetry at fixed $(m,S)$ together with the strict convexity induced by $\ell_2$
regularization: samples sharing the same label set contribute identically, and the minimizer assigns them a common prototype,
yielding $h_{m,S,i}=h_{m,S}$ (Lemma~\ref{lem3}).

Third, at multiplicity one, the two-level tightness condition of the PAL affine bound, together with permutation symmetry, shift invariance, and the first-order optimality condition with respect to \(h_{1,k}\), yields the centered self-duality relation
\(\widehat h_{1,k}=C_1w_k\), equivalently
\(\widehat h_{1,k}=C_1\sqrt{r_{1,k}}u_k\).
This identifies the multiplicity-one directions as the basic building blocks. We do not interpret this as asserting that the rescaled vectors \(\{u_k\}_{k=1}^K\) form an exact centered ETF under general multiplicity-one imbalance.

Finally, for $m\ge 2$, the same equality mechanism must hold \emph{simultaneously} across multiplicities under a shared classifier $W$.
Compatibility of the two-level logit structure with the already fixed multiplicity-one directions forces every higher-multiplicity prototype
to lie in the span of $\{u_k\}$ and yields the synthesis rule
\(
h_{m,S}=C_m\sum_{k\in S}\sqrt{r_{1,k}}\,u_k,
\)
which completes Theorem~\ref{thm:structure-geometry} (proved in Lemma~\ref{lem3}).

\smallskip
\noindent\textbf{Summary of the guiding ideas.}
The proof is driven by two principles: (a) reduce the PAL-CE nonlinearity to a linear, additively separable form (Lemma~\ref{lem:your-lemma}),
so that combinatorial sums collapse to a small set of mean-difference terms; (b) compress general correlation and imbalance into
spectral and counting controls via centering and a Rayleigh/trace route, thereby introducing $\kappa_m$ and the worst-set rarity term.
The structural terminal relations are then recovered by tracking equality conditions across the inequality chain.

\section{Extended Interpretations of the Main Theorems}
\label{app:interpretation}

\subsection{Scope and Organization}
\label{app:interpretation:scope}
This appendix provides extended interpretation and intuition for our two main theorems.
To meet the main-body page limit, the remarks following Theorem~\ref{thm:lower-bound-lemma2-form} and Theorem~\ref{thm:structure-geometry}
are intentionally brief. Here we expand on their motivation, clarify how the key constants and statistics control the bounds and terminal geometry,
and discuss representative edge cases that help connect the theory to practice.

The appendix is organized as follows.
Section~\ref{app:interpretation:thm1} expands the interpretation of Theorem~\ref{thm:lower-bound-lemma2-form} (the distribution-aware attainability bound),
including the role of the shift $g(WH)-\Gamma_2$, the tuning family $\{c_{1,m}\}$, and the distribution-sensitive control quantity $A_m$
with its spectral constant $\kappa_m$ and worst-set rarity amplification term.
Section~\ref{app:interpretation:thm2} expands the interpretation of Theorem~\ref{thm:structure-geometry} (terminal structure and geometry),
highlighting the centering--collapse--self-duality--generation chain, its connection to the multiplicity-one imbalance conjecture in \citet{li2023neural},
and the reduction to the classical tag-wise averaging picture in the balanced special case.

\subsection{Extended Interpretation of Theorem~\ref{thm:lower-bound-lemma2-form}}
\label{app:interpretation:thm1}

Theorem~\ref{thm:lower-bound-lemma2-form} is stated in terms of the shifted quantity $g(WH)-\Gamma_2$ rather than $g(WH)$.
This is not cosmetic: the proof begins with an affine (linear-plus-constant) pointwise lower bound of the PAL loss at each
multiplicity level. After averaging over samples, the intercept terms inevitably accumulate, and their weighted aggregate is
precisely $\Gamma_2$. Subtracting $\Gamma_2$ removes this unavoidable constant shift and isolates the part of the empirical risk
whose behavior is genuinely governed by geometry, distributional complexity, and regularization.

A key technical knob in the affine bound is the collection of positive constants $\{c_{1,m}\}_{m=1}^M$.
For each $m\in\{1,\dots,M\}$, $c_{1,m}>0$ is chosen \emph{arbitrarily}, and the family
$(c_{1,1},c_{1,2},\dots,c_{1,M})$ should be viewed as $M$ independent parameters (one per multiplicity level), rather than a
single global constant. Their role is the same as the familiar ``$c_1$'' in \cite{Zhu2021} or ``$t_k$'' in \cite{dang2024neural} parameters in related NC analyses:
they generate a \emph{family} of valid affine lower bounds, and in principle one may further optimize over them to tighten the final bound.

The slope of the affine PAL lower bound is
\(
\gamma_{1,m}
=
\frac{1}{1+c_{1,m}}\cdot\frac{m}{K-m},
\)
which depends only on $(K,m,c_{1,m})$ and is independent of the data distribution.
In the derivation of the PAL affine bound, $\gamma_{1,m}$ is calibrated so that the first-order (Taylor) bound admits a uniform
linear form in the contrast direction $\mathbf{1}-\tfrac{K}{m}\mathbf{1}_S$.
The prefactor $\tfrac{1}{1+c_{1,m}}$ plays the same tightness-control role as in the single-label CE lower bound (e.g., \cite{Zhu2021}),
while $\tfrac{m}{K-m}$ is the multi-label generalization of the familiar $\tfrac{1}{K-1}$ factor, reflecting the in-/out-group sizes.

The corresponding intercept $c_{2,m}$ is likewise distribution-free, depending only on $(K,m,c_{1,m})$.
Aggregating these intercepts yields
\(
\Gamma_2=\frac{1}{N}\sum_{m=1}^M N_m\,c_{2,m}.
\)
Thus, once $(K,M)$ and the chosen $\{c_{1,m}\}$ are fixed, $\Gamma_2$ acts as a comparatively stable shift:
its data dependence enters only through the multiplicity counts $\{N_m\}$ as weights, and it does not carry the distribution-sensitive
geometry that the theorem is designed to expose.

The distribution-sensitive content is concentrated in $A_m$. 
Three features are worth emphasizing.
First, $A_m$ grows as $\kappa_m$ decreases: the factor $1/\kappa_m$ captures spectral near-degeneracy on centered contrast directions.
Small $\kappa_m$ is not merely ``few samples''; it indicates that the centered co-occurrence structure fails to excite some contrast direction,
so any spectral lower bound necessarily becomes ill-conditioned, and the overall bound deteriorates accordingly.
Second, $A_m$ decreases with $N_m$, reflecting the benefit of having more samples in multiplicity $m$.
Third, within-group rarity enters through the worst-set term
\[
\max_{S\in\mathcal S_m}\sum_{j\in S}\frac{1}{N_m^j},
\]
which selects a label set containing rare labels and accumulates their rarity linearly. Hence, even when $N_m$ is large, a very small
$N_m^j$ can still significantly enlarge $A_m$ via the worst-case set; this provides a direct statistical mechanism by which long-tailed
labels are amplified at the combinatorial level. This ``worst-set amplification'' parallels the frequency-driven scaling effects observed
in imbalanced NC analyses for multi-class learning (e.g., \cite{dang2024neural}), where imbalance enters the optimal value/geometry through explicit
count statistics rather than a mere constant shift.

The remaining factor on the right-hand side,
\(
\sqrt{\frac{\lambda_W}{\lambda_H}}\,\rho, 
\)
\(
\rho=\|W\|_F^2,
\)
is an energy scale induced by $\ell_2$ regularization and closed via the critical-point scaling identity
$\|H\|_F^2=(\lambda_W/\lambda_H)\|W\|_F^2$ (see Lemma~\ref{lem:a1} for proof).
This closure is what makes the cross-multiplicity aggregation possible and also explains why $A_m$ is defined to avoid any additional
dependence on $\rho$.

Putting these pieces together, the inequality
\[
g(WH)-\Gamma_2
\ge
-\frac{1}{N}\sum_{m=1}^M N_m\,\gamma_{1,m}\,A_m\,
\sqrt{\frac{\lambda_W}{\lambda_H}}\,
\rho
\]
should be read as a distribution-aware attainability bound: after removing the distribution-free affine shift $\Gamma_2$,
the optimal risk cannot drop arbitrarily; its worst-case negative deviation is governed by the tunable affine slopes
$\{\gamma_{1,m}\}$ and, more importantly, by the distributional complexity encoded in $(\kappa_m,\{N_m^j\})$ through $A_m$.
Mathematically, this bound compresses correlation and imbalance into computable scalar controls.
From a data perspective, it implies that even with large overall sample sizes, near-degeneracy of centered co-occurrence
(small $\kappa_m$) or extreme within-multiplicity rarity (small $N_m^j$) can substantially weaken provable guarantees and
destabilize the terminal geometry.

\subsection{Extended Interpretation of Theorem~\ref{thm:structure-geometry}}
\label{app:interpretation:thm2}

Theorem~\ref{thm:structure-geometry} summarizes the terminal geometry exhibited by any global minimizer $(W,H)$ in the
(neural-collapse) regime. The four statements are best read as four facets of a single terminal structure: a centering
geometry for the classifier, a collapse to label-set prototypes, a centered self-duality relation at multiplicity one, and a
generation law that produces all higher-multiplicity prototypes from the multiplicity-one directions.

The classifier-centering phenomenon,
\[
\sum_{k=1}^K w_k=0,
\qquad
\Pi W=W,
\]
indicates that the terminal discriminative geometry lives in the centered subspace. The $\mathbf{1}$-direction corresponds to a
global shift of logits and does not affect relative comparisons; removing this redundant degree of freedom makes the relevant
geometric relations intrinsically about centered contrast directions.

The within-group collapse phenomenon,
\[
h_{m,S,i}=h_{m,S},
\qquad
i=1,\dots,r_{m,S},
\]
is the multi-label analogue of variability collapse, but at the natural granularity of label sets: within a fixed multiplicity
$m$, the label set $S$ is the relevant grouping, and the terminal representation assigns a single prototype $h_{m,S}$ to each
group.

Turning to multiplicity one, the theorem identifies a centered self-duality relation between the multiplicity-one prototypes
and the classifier directions. Define
\[
\widehat h_{1,k}
=
h_{1,k}
-
\frac{1}{K}\sum_{\ell=1}^K h_{1,\ell},
\qquad
u_k=\frac{w_k}{\sqrt{r_{1,k}}}.
\]
Then statement (3) gives
\[
\widehat h_{1,k}=C_1 w_k,
\qquad\text{equivalently}\qquad
\widehat h_{1,k}=C_1\sqrt{r_{1,k}}\,u_k.
\]
Thus, after centering, each multiplicity-one prototype is aligned with its corresponding classifier direction, while the
multiplicity-one count profile enters explicitly through the factor $\sqrt{r_{1,k}}$ in the imbalance-aware coordinates.
More precisely, statement (3) identifies the multiplicity-one directions as the basic building blocks, and statement (4)
shows that the higher-multiplicity geometry is generated from them through explicit count-dependent weights. We do not
interpret this as asserting that the rescaled vectors $\{u_k\}_{k=1}^K$ necessarily form an exact centered ETF under general
multiplicity-one imbalance.

The most distinctive outcome is the label-set generation law for $m\ge 2$,
\[
h_{m,S}=C_{m}\sum_{k\in S}\sqrt{r_{1,k}}\,u_k
\qquad
\left(
\text{equivalently, } h_{m,S}=C_{m}\sum_{k\in S} w_k
\right).
\]
This identity is more precise than a naive averaging statement: it expresses every higher-multiplicity prototype as a linear
combination of the multiplicity-one directions, with explicit count-dependent weights determined by
$\{\sqrt{r_{1,k}}\}$. Thus, the geometry of terminal higher-multiplicity is anchored by the multiplicity-one configuration and
is propagated to all $m\ge 2$ through a single synthesis rule, with multiplicity-one imbalance acting as the structural driver
of the combination weights.

This view of synthesis directly addresses the conjectured imbalance regime discussed by previous work~\cite{li2023neural}. In
their conclusion, they highlighted the challenging case where multiplicity-one training data are imbalanced and conjectured
that a scaled-average relationship should still persist between higher-multiplicity features and their multiplicity-one
features. Theorem~\ref{thm:structure-geometry} confirms this picture in a provable form and makes the imbalance dependence
explicit: uniform averaging is replaced by the $\sqrt{r_{1,k}}$-weighted generation law above.

Finally, the theorem recovers the symmetric within-multiplicity setting of~\cite{li2023neural} as a special case. When, for
each $m$, the counts are uniform within $\mathcal S_m$ (i.e., $r_{m,S}\equiv n_m$) and multiplicity one is balanced (so
$r_{1,k}$ does not vary with $k$), the weights $\sqrt{r_{1,k}}$ become constant and the generation law reduces to an
unweighted tag-wise summation (and, after normalization, to the scaled-average structure in~\cite{li2023neural}). Hence
Theorem~\ref{thm:structure-geometry} reproduces the balanced-within-multiplicity picture while extending it to the regime of
multiplicity-one imbalance and more general within-multiplicity non-uniformity.

\section{Proofs of the Main Results}
\label{appendix1}

This appendix provides the complete proofs of the main results stated in the paper.
In particular, the two lemmas below are the \emph{formal, technical} versions of the two main theorems presented in the main text:
they make all assumptions explicit, specify constants precisely, and are stated in a form that is directly amenable to the step-by-step proof arguments.
Theorems~\ref{thm:lower-bound-lemma2-form} and~\ref{thm:structure-geometry} in the main text follow immediately from Lemmas~\ref{lem:lemma2_lb} and~\ref{lem3}, respectively,
by straightforward specialization and notational simplification.

\begin{lemma}[Lower bound via $\kappa_m$ and the $\Theta$--$H$ interface]\label{lem:lemma2_lb}
Let $(W,H)$ be a critical point of the objective in Objective~\ref{eq:objective}.
Fix any $c_{1,m}>0$ and define
\(
\gamma_{1,m}:=\frac{1}{1+c_{1,m}}\cdot\frac{m}{K-m}.
\)
Let
\(
\Pi:=I_K-\frac{1}{K}\mathbf{1}\mathbf{1}^{\top},
G_m:=\mathbb{E}_{S\sim p_m}\!\left[y_Sy_S^{\top}\right],
\)
and define the weighted spectral constant
\(
\kappa_m := \lambda_{\min}\Bigl( (\Pi G_m \Pi)\big|_{\mathrm{range}(\Pi)} \Bigr),\;
\rho:=\|W\|_F^2,\;
and \; \Gamma_2:=\frac{1}{N}\sum_{m=1}^{M}N_m\,c_{2,m}
\),
where $c_{2,m}$ depends only on $(K,m,c_{1,m})$ as in Lemma~\ref{lem:your-lemma}.
Meanwhile, \(\kappa_m\) satisfies Assumption~\ref{assump:nondeg-centered}. 
Then
\[
g(WH)-\Gamma_2
\ \ge\
-\frac{1}{N}\sum_{m=1}^{M} N_m\,\gamma_{1,m}\,A_m\,\sqrt{\frac{\lambda_W}{\lambda_H}}\;\rho,
\]
where
\[
A_m:=
\sqrt{
\frac{1}{\kappa_m}\cdot \frac{m(K-m)}{K}\cdot
\Biggl[
\frac{2K}{N_m}
+
\frac{2K^{2}}{m^{2}}
\Biggl(
\max_{|S|=m}\sum_{j=1}^{K}\frac{I_S(j)}{N_m^{j}}
\Biggr)
\Biggr]
}.
\]
Moreover, $\Gamma_2$ depends on $\{c_{1,m}\}_{m=1}^{M}$ because $\{c_{2,m}\}_{m=1}^{M}$  depends on $\{c_{1,m}\}_{m=1}^{M}$ for each m.
\end{lemma}

\begin{proof}
Throughout this proof, let
\(
z_{m,S,i}:=Wh_{m,S,i},
\)
and choose the same $\gamma_{1,m}$ and $c_{2,m}$ for all $i$ and $k$.
Grouping by multiplicity,
\begin{equation}\label{eq:23B}
g_m(WH_m)
=
\frac{1}{N_m}
\sum_{\substack{S\subset [K]\\ |S|=m}}
\sum_{i=1}^{r_{m,S}}
\mathcal{L}_{\mathrm{PAL}}(z_{m,S,i},y_S).
\end{equation}

By applying Lemma~\ref{lem:your-lemma} to each term in \eqref{eq:23B}, we obtain
\begin{equation}\label{eq:q7}
\begin{aligned}
N_m g_m(WH_m)
&\ \ge\
\gamma_{1,m}
\sum_{\substack{S\subset [K]\\ |S|=m}}
\sum_{i=1}^{r_{m,S}}
\Bigl\langle \mathbf{1}-\frac{K}{m}\mathbf{1}_S,\ Wh_{m,S,i}\Bigr\rangle +N_m c_{2,m},
\end{aligned}
\end{equation}
which implies that
\begin{equation}
\begin{aligned}
\label{eq:26}
\gamma_{1,m}^{-1}(g_m(WH_m)-c_{2,m})
&\ \ge\
\frac{1}{N_m}
\sum_{\substack{S\subset [K]\\ |S|=m}}
\sum_{i=1}^{r_{m,S}}
\Bigl\langle \mathbf{1}-\frac{K}{m}\mathbf{1}_S,\ Wh_{m,S,i}\Bigr\rangle .
\end{aligned}
\end{equation}

For any $a\in\mathbb{R}^K$ and $z\in\mathbb{R}^K$,
\begin{equation}\label{eq:inner_prod_expand}
\langle a,z\rangle=\sum_{j=1}^{K}a_j z_j = a^{\top}z = z^{\top}a.
\end{equation}
Let $a=\mathbf{1}-\frac{K}{m}\mathbf{1}_S$ and $z=Wh_{m,S,i}\in\mathbb{R}^K$. Then
\begin{equation}\label{eq:lem2_inner_prod_sumj}
\Bigl\langle \mathbf{1}-\frac{K}{m}\mathbf{1}_S,\ Wh_{m,S,i}\Bigr\rangle
=
\sum_{j=1}^{K}\Bigl(1-\frac{K}{m}I_S(j)\Bigr)\,(Wh_{m,S,i})_j .
\end{equation},

where $(\cdot)_j$ denotes the $j$-th row of a vector, \( I_S(j):=\begin{cases}
1, & \text{if } j\in S,\\
0, & \text{if } j\notin S.
\end{cases} \)
and 

Moreover, from \eqref{eq:inner_prod_expand}
\begin{equation}\label{eq:lem2_row_col_product}
(Wh_{m,S,i})_j = h_{m,S,i}^{\top}w^j,
\end{equation}

Substituting \eqref{eq:lem2_inner_prod_sumj}--\eqref{eq:lem2_row_col_product} into the right-hand side of
\eqref{eq:inner_prod_expand} yields
\begin{align}
\frac{1}{N_m}
\sum_{\substack{S\subset [K]\\ |S|=m}}
\sum_{i=1}^{r_{m,S}}
\Bigl\langle \mathbf{1}-\frac{K}{m}\mathbf{1}_S,\ Wh_{m,S,i}\Bigr\rangle
&=
\frac{1}{N_m}
\sum_{\substack{S\subset [K]\\ |S|=m}}
\sum_{i=1}^{r_{m,S}}
\sum_{j=1}^{K}
\Bigl(1-\frac{K}{m}I_S(j)\Bigr)\,h_{m,S,i}^{\top}w^j
\label{eq:27B}\\
&=
\sum_{j=1}^{K}
\frac{1}{N_m}
\sum_{\substack{S\subset [K]\\ |S|=m}}
\sum_{i=1}^{r_{m,S}}
\Bigl(1-\frac{K}{m}I_S(j)\Bigr)\,h_{m,S,i}^{\top}w^j
\label{eq:28B}\\
&=
\sum_{j=1}^{K}
\Biggl(
\frac{1}{N_m}
\sum_{\substack{S\subset [K]\\ |S|=m}}
\sum_{i=1}^{r_{m,S}}
\Bigl(1-\frac{K}{m}I_S(j)\Bigr)\,h_{m,S,i}
\Biggr)^{\top}w^j .
\label{eq:29B}
\end{align}

Define
\[
\bar h_m
:=
\frac{1}{N_m}
\sum_{\substack{S\subset [K]\\ |S|=m}}
\sum_{i=1}^{r_{m,S}} h_{m,S,i},
\]
and
\[
\bar h_m^{\,j}
:=
\frac{1}{N_m}\sum_{\substack{S:\ j\in S\\ S\subset [K],\ |S|=m}}
\sum_{i=1}^{r_{m,S}} h_{m,S,i}.
\]



For any fixed $j\in[K]$, we have
\[
\frac{1}{N_m}\sum_{\substack{S\subset [K]\\ |S|=m}}\sum_{i=1}^{r_{m,S}}
\Bigl(1-\frac{K}{m}I_S(j)\Bigr)\,h_{m,S,i}
=
\bar h_m-\frac{K}{m}\bar h_m^{\,j}.
\]
Substituting the above identity into \eqref{eq:29B} yields
\begin{equation}
\label{eq:hmw-expand}
\frac{1}{N_m}\sum_{\substack{S\subset [K]\\ |S|=m}}\sum_{i=1}^{r_{m,S}}\Bigl\langle \mathbf{1}-\frac{K}{m}\mathbf{1}_S,\ Wh_{m,S,i}\Bigr\rangle =
\sum_{j=1}^{K}(\bar h_m-\frac{K}{m}\bar h_m^{\,j}
)^{\top}\,w^{j}\ .
\end{equation}

To avoid notational confusion with the set index $j$ in $I_S(j)$, we relabel the class index by $k\in[K]$,
and define
\[
\bar h_m^{\,k}
:=
\frac{1}{N_m}
\sum_{\substack{S\subset [K]:\,k\in S\\ |S|=m}}
\sum_{i=1}^{r_{m,S}} h_{m,S,i}.
\]

By \eqref{eq:26} and \eqref{eq:hmw-expand}, we obtain
\begin{equation}
\label{eq:gm-lower-hmw}
\gamma_{1,m}^{-1}\bigl(g_m(WH_m)-c_{2,m}\bigr)
\ge
\sum_{k=1}^{K}\Bigl(\bar h_m-\frac{K}{m}\bar h_m^{\,k}\Bigr)^{\top}w^{k}.
\end{equation}

Next we apply Young's inequality. Take any constant $c_{3,m}>0$. For each $k\in[K]$, let
\[
a_k:=\bar h_m-\frac{K}{m}\bar h_m^{\,k},
\qquad
b_k:=w^{k}.
\]
Young's inequality gives
\[
a_k^{\top}b_k
\ge
-\frac{c_{3,m}}{2}\|b_k\|^{2}
-\frac{1}{2c_{3,m}}\|a_k\|^{2},
\]
that is,
\[
\Bigl(\bar h_m-\frac{K}{m}\bar h_m^{\,k}\Bigr)^{\top}w^{k}
\ge
-\frac{c_{3,m}}{2}\|w^{k}\|^{2}
-\frac{1}{2c_{3,m}}
\Bigl\|\bar h_m-\frac{K}{m}\bar h_m^{\,k}\Bigr\|^{2}.
\]
Summing over $k=1,\dots,K$ and using $\sum_{k=1}^{K}\|w^{k}\|^{2}=\|W\|_{F}^{2}$, we obtain
\begin{equation}
\label{eq:young-sum}
\sum_{k=1}^{K}\Bigl(\bar h_m-\frac{K}{m}\bar h_m^{\,k}\Bigr)^{\top}w^{k}
\ge
-\frac{c_{3,m}}{2}\|W\|_{F}^{2}
-\frac{1}{2c_{3,m}}\sum_{k=1}^{K}
\Bigl\|\bar h_m-\frac{K}{m}\bar h_m^{\,k}\Bigr\|^{2}.
\end{equation}
Substituting \eqref{eq:young-sum} into \eqref{eq:gm-lower-hmw} yields
\begin{equation}
\label{eq:gm-lower-young}
\gamma_{1,m}^{-1}\bigl(g_m(WH_m)-c_{2,m}\bigr)
\ge
-\frac{c_{3,m}}{2}\|W\|_{F}^{2}
-\frac{1}{2c_{3,m}}\sum_{k=1}^{K}
\Bigl\|\bar h_m-\frac{K}{m}\bar h_m^{\,k}\Bigr\|^{2}.
\end{equation}
Define \(\Theta_m\in\mathbb R^{K\times d}\) row-wise by
\[
\Theta_m(k,:)
:=
\left(
\bar h_m-\frac{K}{m}\bar h_m^{\,k}
\right)^\top,
\qquad k=1,\dots,K.
\]
Then
\[
\sum_{k=1}^K
\left\|
\bar h_m-\frac{K}{m}\bar h_m^{\,k}
\right\|_2^2
=
\|\Theta_m\|_F^2.
\]

Substituting into \eqref{eq:gm-lower-young} gives:

\[
\gamma_{1,m}^{-1}\bigl(g_m(WH_m)-c_{2,m}\bigr)
\ge
-\frac{c_{3,m}}{2}\|W\|_{F}^{2}
-\frac{1}{2c_{3,m}}\left\| \Theta_m \right\|_F^2.
\]

Let the indicator (all-ones) vector
\[
\mathbf{1}:=(1,\cdots,1)^{\top}\in\mathbb{R}^{K},
\qquad
\Pi:=I-\frac{1}{K}\mathbf{1}\mathbf{1}^{\top}.
\]

We now prove
\begin{equation}
\label{eq:hengdengshi}
\mathbf 1^\top\Theta_m=0
\qquad\Longleftrightarrow\qquad
\Pi\Theta_m=\Theta_m .
\end{equation}

Indeed,
\begin{equation}
\label{eq:sum1}
\mathbf 1^\top\Theta_m
=
\sum_{k=1}^{K}
\left(
\bar h_m-\frac{K}{m}\bar h_m^{\,k}
\right)^\top
=
\left(
K\bar h_m-\frac{K}{m}\sum_{k=1}^{K}\bar h_m^{\,k}
\right)^\top .
\end{equation}

and
\[
\sum_{k=1}^{K}\bar h_m^{\,k}
=
\frac{1}{N_m}\sum_{k=1}^{K}\ \sum_{\substack{S\in\mathcal S_m\\ k\in S}}\ \sum_{i=1}^{r_{m,S}} h_{m,S,i}
=
\frac{1}{N_m}\sum_{S\in\mathcal S_m}\ \sum_{i=1}^{r_{m,S}}
\Bigl(\sum_{k=1}^{K} I\{k\in S\}\Bigr)h_{m,S,i}.
\]

Since \(|S|=m\), we have
\[
\sum_{k=1}^{K} I\{k\in S\}=m.
\]
Thus
\begin{equation}
\label{eq:sum2}
\sum_{k=1}^{K}\bar h_m^{\,k}
=
\frac{1}{N_m}\sum_{S,i} m\,h_{m,S,i}
=
m\bar h_m.
\end{equation}

Substituting \eqref{eq:sum2} into \eqref{eq:sum1} gives
\[
\mathbf 1^\top\Theta_m=0.
\]
Hence \(\Pi\Theta_m=\Theta_m\), since
\[
\Pi\Theta_m
=
\Theta_m-\frac{1}{K}\mathbf 1\mathbf 1^\top\Theta_m
=
\Theta_m.
\]

Next, by Lemma~\ref{lem:s3}, for any \(Z\in\mathbb{R}^{K\times d}\),
\[
\operatorname{Tr}\!\bigl(Z^{\top}\Pi G_m\Pi Z\bigr)\ge \kappa_m\|\Pi Z\|_{F}^{2}.
\]

Now take \(Z=\Theta_m\in\mathbb R^{K\times d}\) in Lemma~\ref{lem:s3}. Since \(\Pi\Theta_m=\Theta_m\), we obtain
\begin{equation}
\label{eq:sepc-ub}
\|\Theta_m\|_{F}^{2}
\le
\frac{1}{\kappa_m}
\operatorname{Tr}\!\bigl(\Theta_m^{\top}\Pi G_m\Pi \Theta_m\bigr).
\end{equation}

Substituting \eqref{eq:sepc-ub} into \eqref{eq:gm-lower-young} yields
\begin{equation}
\label{eq:after-s3}
\gamma_{1,m}^{-1}\bigl(g_m(WH_m)-c_{2,m}\bigr)
\ge
-\frac{c_{3,m}}{2}\|W\|_{F}^{2}
-\frac{1}{2c_{3,m}\cdot \kappa_m}\operatorname{Tr}\!\bigl(\Theta_m^{\top}\Pi G_m\Pi \Theta_m\bigr),
\end{equation}

Next, from \eqref{eq:after-s3} we directly obtain
\begin{equation}
\label{eq:k1}
g_m(WH_m)-c_{2,m}
\ge
-\gamma_{1,m}\Biggl[
\frac{c_{3,m}}{2}\rho
+
\frac{1}{2c_{3,m}\cdot \kappa_m}
\operatorname{Tr}\!\bigl(\Theta_m^{\top}\Pi G_m\Pi \Theta_m\bigr)
\Biggr].
\end{equation}

Let \( T_m := \frac{1}{\kappa_m} \operatorname{Tr}\bigl(\Theta_m^{\top} \Pi G_m \Pi \Theta_m\bigr) \),

then \eqref{eq:k1} satisfies
\begin{equation}
\label{eq:k2}
g_m(WH_m)-c_{2,m} \geq -\gamma_{1,m}(\frac{c_{3,m}}{2} \rho + \frac{1}{2c_{3,m} } T_m)
.
\end{equation}
Next, define
\[
\Phi(c) := \frac{c}{2} \rho + \frac{1}{2c} T, \quad c > 0,
\]
and compute its derivative
\[
\Phi'(c) = \frac{\rho}{2} - \frac{T}{2c^2}.
\]

Setting \( \Phi'(c) = 0 \), we get
\[
c^2 = \frac{T}{\rho},
\]

Moreover, since 
\[
\Phi''(c)=\frac{T}{c^3}>0 
\]
Therefore, it is a global minimizer.

Then, substituting the stationary point, we obtain
\[
\Phi(c^*) = \frac{1}{2} (\sqrt{\frac{T}{\rho}}\rho + \frac{T}{\sqrt{T/{\rho}}})= \sqrt{\rho T}.
\]

Thus, from \eqref{eq:k2}, we obtain
\begin{equation}
\label{eq:k3}
g_m(WH_m)-c_{2,m} \geq -\gamma_{1,m}\sqrt{\rho}\sqrt{T_m}
.
\end{equation}

By definition,
\[
G_m := \mathbb{E}_{S\sim p_m}\bigl[y_S y_S^{\top}\bigr],\qquad y_S=\mathbf{1}_S,
\]
and by the linearity and cyclic invariance of the trace,
\[
\mathrm{Tr}(\Pi G_m \Pi)
=
\mathrm{Tr}\!\Bigl(\Pi\,\mathbb{E}[y_S y_S^{\top}]\,\Pi\Bigr)
=
\mathbb{E}\bigl[\mathrm{Tr}(\Pi y_S y_S^{\top}\Pi)\bigr].
\]

By cyclic invariance and \(\mathrm{Tr}(uv^{\top})=v^{\top}u\), and using \(\Pi^{\top}=\Pi,\ \Pi^{2}=\Pi\),
\[
\mathrm{Tr}(\Pi y_S y_S^{\top}\Pi)
=
\mathrm{Tr}\bigl((\Pi y_S)(\Pi y_S)^{\top}\bigr)
=
\|\Pi y_S\|^{2}
=
y_S^{\top}\Pi^{\top}\Pi y_S
=
y_S^{\top}\Pi y_S.
\]

Therefore,
\begin{equation}
\label{eq:k5}
\mathrm{Tr}(\Pi G_m \Pi)=\mathbb{E}\bigl[y_S^{\top}\Pi y_S\bigr].
\end{equation}

Next we explicitly compute \(y_S^{\top}\Pi y_S\). Substituting
\[
\Pi=I-\frac{1}{K}\mathbf{1}\mathbf{1}^{\top},
\]
\begin{equation}
\label{eq:k4}
y_S^{\top}\Pi y_S
=
y_S^{\top}y_S-\frac{1}{K}y_S^{\top}\mathbf{1}\mathbf{1}^{\top}y_S
=
\|y_S\|^{2}-\frac{1}{K}(\mathbf{1}^{\top}y_S)^{2}.
\end{equation}

Since \(y_S\) is the indicator vector of the set \(S\) and \(|S|=m\), we have
\[
\|y_S\|^{2}=m,\qquad \mathbf{1}^{\top}y_S=m.
\]

Substituting into \eqref{eq:k4} yields
\[
y_S^{\top}\Pi y_S
=
m-\frac{m^{2}}{K}.
\]
Since this expression is the same for all \(|S|=m\), the expectation does not change.

Thus, by \eqref{eq:k5},
\begin{equation}
\label{eq:k6}
\operatorname{Tr}(\Pi G_m \Pi)=m-\frac{m^{2}}{K}=\frac{m(K-m)}{K}.
\end{equation}

Next, let \(B=\Pi G_m \Pi\). Then by Lemma~\ref{lem:trace-ineq},
\begin{equation}
\label{eq:k7}
\operatorname{Tr}\!\bigl(\Theta_m^{\top}\Pi G_m \Pi \Theta_m\bigr)
\le
\operatorname{Tr}(\Pi G_m \Pi)\,\|\Theta_m\|_{F}^{2}.
\end{equation}

By \eqref{eq:k6}, we further have
\[
\operatorname{Tr}\!\bigl(\Theta_m^{\top}\Pi G_m \Pi \Theta_m\bigr)
\le
\frac{m(K-m)}{K}\,\|\Theta_m\|_{F}^{2}.
\]

Dividing both sides by \(\kappa_m\) yields
\begin{equation}
\label{eq:k8}
T_m=\frac{1}{\kappa_m}\operatorname{Tr}\!\bigl(\Theta_m^{\top}\Pi G_m \Pi \Theta_m\bigr)
\le
\frac{1}{\kappa_m}\cdot \frac{m(K-m)}{K}\cdot \|\Theta_m\|_{F}^{2}.
\end{equation}
Substituting \eqref{eq:k8} into \eqref{eq:k3} yields
\begin{equation}
\label{eq:k9}
g_m(WH_m)-c_{2,m}
\ge
-\gamma_{1,m}\cdot \sqrt{\rho}\cdot
\sqrt{\frac{m(K-m)}{K}\cdot \frac{1}{\kappa_m}}\cdot
\|\Theta_m\|_{F}.
\end{equation}

By Lemma~\ref{lem:S5_theta_H_interface},
\begin{equation}
\label{eq:k10}
\|\Theta_m\|_{F}^{2}
\le
\Biggl[
\frac{2K}{N_m}
+
\frac{2K^{2}}{m^{2}}
\Biggl(
\max_{|S|=m}\sum_{j=1}^{K}\frac{I_S(j)}{N_m^{j}}
\Biggr)
\Biggr]
\sum_{|S|=m}\sum_{i=1}^{r_{m,S}}\|h_{m,S,i}\|^{2}.
\end{equation}

Define
\[
C_m:=
\frac{2K}{N_m}
+
\frac{2K^{2}}{m^{2}}
\Biggl(
\max_{|S|=m}\sum_{j=1}^{K}\frac{I_S(j)}{N_m^{j}}
\Biggr).
\]

Note that \(C_m\) is determined only by the counts \(N_m,\ N_m^{j},\ r_{m,S},\ K,\ m\), and does not involve \(W\), nor the directional information of \(h\).

Then we define
\[
E_m:=
\sum_{|S|=m}\sum_{i=1}^{r_{m,S}}\|h_{m,S,i}\|^{2}.
\]

Then \eqref{eq:k10} can be simplified as
\begin{equation}
\label{eq:k11}
\|\Theta_m\|_{F}^{2}\le C_m E_m.
\end{equation}

Substituting \eqref{eq:k11} into \eqref{eq:k9} yields
\begin{equation}
\label{eq:k12}
g_m(WH_m)-c_{2,m}
\ge
-\gamma_{1,m}\cdot \sqrt{\rho}\cdot
\sqrt{\frac{1}{\kappa_m}\cdot \frac{m(K-m)}{K}\cdot C_m\cdot E_m}.
\end{equation}

Define
\[
A_m:=
\sqrt{
\frac{1}{\kappa_m}\cdot \frac{m(K-m)}{K}\cdot
\Biggl[
\frac{2K}{N_m}
+
\frac{2K^{2}}{m^{2}}
\Biggl(
\max_{|S|=m}\sum_{j=1}^{K}\frac{I_S(j)}{N_m^{j}}
\Biggr)
\Biggr]
}.
\]

Expanding \(C_m\), from \eqref{eq:k12}, for each \(m\),
\[
g_m(WH_m)-c_{2,m}
\ge
-\gamma_{1,m}A_m\cdot \sqrt{\rho}\cdot
\sqrt{\sum_{|S|=m}\sum_{i=1}^{r_{m,S}}\|h_{m,S,i}\|^{2}}.
\]

Note that \(A_m\) does not contain \(\rho\), and depends only on \((K,m)\), the counting structure \((N_m,\ N_m^{j},\ r_{m,S})\), and the spectral constant \(\kappa_m\).

By definition,
\[
g(WH):=\frac{1}{N}\sum_{m=1}^{M}N_m g_m(WH_m),
\qquad
\Gamma_2:=\frac{1}{N}\sum_{m=1}^{M}N_m c_{2,m}.
\]

Therefore,
\begin{equation}
\label{eq:k13}
g(WH)-\Gamma_2
=
\frac{1}{N}\sum_{m=1}^{M}N_m\bigl(g_m(WH_m)-c_{2,m}\bigr).
\end{equation}

Multiplying \eqref{eq:k12} by \(N_m\), summing over \(m\), and dividing by \(N\), we get
\begin{equation}
\label{eq:k14}
g(WH)-\Gamma_2
\ge
-\frac{1}{N}\sum_{m=1}^{M}N_m\gamma_{1,m}A_m\sqrt{\rho}\cdot
\sqrt{\sum_{|S|=m}\sum_{i=1}^{r_{m,S}}\|h_{m,S,i}\|^{2}}.
\end{equation}

Let
\begin{equation}
\label{eq:k15}
\|H\|_{F}^{2}
:=
\sum_{m=1}^{M}\sum_{|S|=m}\sum_{i=1}^{r_{m,S}}\|h_{m,S,i}\|^{2}
\ge
\sum_{|S|=m}\sum_{i=1}^{r_{m,S}}\|h_{m,S,i}\|^{2}
=:E_m.
\end{equation}

Hence,
\[
\sqrt{E_m}\le \|H\|_{F}.
\]

Substituting \eqref{eq:k15} into \eqref{eq:k14} yields
\begin{equation}
\label{eq:k16}
g(WH)-\Gamma_2
\ge
-\frac{1}{N}\sum_{m=1}^{M}N_m\gamma_{1,m}A_m\sqrt{\rho}\,\|H\|_{F}.
\end{equation}

By Lemma~\ref{lem:a1}, we have the identity
\[
\|H\|_{F}^{2}
=
\frac{\lambda_W}{\lambda_H}\|W\|_{F}^{2}
=
\frac{\lambda_W}{\lambda_H}\rho .
\]

Hence
\begin{equation}
\label{eq:k17}
\|H\|_{F}
=
\sqrt{\frac{\lambda_W}{\lambda_H}}\sqrt{\rho}.
\end{equation}

Substituting \eqref{eq:k17} into \eqref{eq:k16} gives
\[
g(WH)-\Gamma_2
\ge
-\frac{1}{N}\sum_{m=1}^{M}N_m\gamma_{1,m}A_m\sqrt{\rho}\cdot
\sqrt{\frac{\lambda_W}{\lambda_H}}\cdot \sqrt{\rho}
=
-\frac{1}{N}\sum_{m=1}^{M}N_m\gamma_{1,m}A_m\sqrt{\frac{\lambda_W}{\lambda_H}}\cdot \rho .
\]
\end{proof}

\begin{lemma}\label{lem3}
Suppose $(W,H)$ is a global minimizer of Problem~\eqref{eq:objective}. Define
\[
\bar h_m:=
\frac{1}{N_m}\sum_{|S|=m}\sum_{i=1}^{r_{m,S}}h_{m,S,i},
\qquad
N_m:=\sum_{|S|=m}r_{m,S}.
\]

Then the following hold:
\begin{enumerate}
\item \textbf{(Centered classifier)} \label{state1}
\[
\sum_{k=1}^{K} w_k = 0.
\]
Equivalently, letting
\[
\Pi = I - \frac{1}{K}\mathbf{1}\mathbf{1}^\top,
\]
we have
\[
\Pi W = W.
\]

\item \textbf{(Within-group collapse)} \label{state2}

For each multiplicity $m$, each label set $S\subset [K]$ with $|S|=m$, and all samples
$i=1,\dots,r_{m,S}$, we have
\[
h_{m,S,i}=h_{m,S},
\]
i.e., all within-group features for the same $(m,S)$ collapse to a single point.

\item \textbf{(Multiplicity-one centered self-duality)} \label{state3}

Define
\[
\widehat h_{1,k}:=
h_{1,k}-\frac{1}{K}\sum_{\ell=1}^{K} h_{1,\ell},
\qquad k\in [K].
\]
Then there exists a constant $C_1>0$ such that
\[
\widehat h_{1,k}=C_1 w_k,
\qquad \forall k\in [K].
\]

\item \textbf{(Multiplicity-one imbalance-aware form)} \label{state4}

Let $r_{1,k}$ be the sample size of class $k$ in the multiplicity-one group, and define
\[
u_k:=\frac{w_k}{\sqrt{r_{1,k}}}.
\]
Then statement~\ref{state3} can be equivalently written as
\[
\widehat h_{1,k}=C_1\sqrt{r_{1,k}}\,u_k,
\qquad \forall k\in [K].
\]

\item \textbf{(Higher-multiplicity generation law)} \label{state5}

For every $m\in\{2,\dots,M\}$ and every label set $S\subset [K]$ with $|S|=m$, there exists a multiplicity-level constant \(C_m>0\), shared across all label sets \(S\in\mathcal S_m\), such that
\[
h_{m,S,i}=h_{m,S}
=
C_m\sum_{k\in S}\sqrt{r_{1,k}}\,u_k.
\]

Equivalently,
\[
h_{m,S,i}=h_{m,S}
=
C_m\sum_{k\in S} w_k.
\]
\end{enumerate}
\end{lemma}

\begin{proof}
Throughout the structural proof, we track the equality conditions of the lower-bound chain in Lemma~\ref{lem:lemma2_lb}. In particular, at the terminal minimizer under consideration, the PAL affine bound in Lemma~\ref{lem:your-lemma} is attained at each collapsed group, so the corresponding logits satisfy the two-level tightness conditions.

(i) Proof of statement \ref{state1}.

Since \(\Pi\) is a projection, we have the orthogonal decomposition:
\[
W=\Pi W+(I-\Pi)W.
\]

Moreover, the two parts are orthogonal under the Frobenius inner product:
\[
\langle \Pi W,(I-\Pi)W\rangle_F=0.
\]

Therefore, we have the Pythagorean identity:
\[
\|W\|_F^{2}=\|\Pi W\|_F^{2}+\|(I-\Pi)W\|_F^{2}.
\]

In particular:
\begin{equation}
\label{eq:q1}
\|\Pi W\|_F\le \|W\|_F
\end{equation}
with equality if and only if \((I-\Pi)W=0\), i.e., \(\Pi W=W\).

By Lemma~\ref{lem:S6_shift_invariance}, the loss is invariant:
\begin{equation}
\label{eq:q2}
g(WH)=g((\Pi W)H).
\end{equation}

Also, by \eqref{eq:q1} we know that the regularization term does not increase.
Combining \eqref{eq:q1} and \eqref{eq:q2}, we obtain
\begin{equation}
\label{eq:projection-objective}
\mathcal{L}(\Pi W,H)\le \mathcal{L}(W,H),
\end{equation}
and if \((I-\Pi)W\ne 0\) then the inequality is strict.

Let \((W^*,H^*)\) be an arbitrary global minimizer. By \eqref{eq:projection-objective},
\[
\mathcal{L}(\Pi W^*,H^*)\le \mathcal{L}(W^*,H^*).
\]

Since \((W^*,H^*)\) is also a global minimizer, equality must hold, hence we must have
\[
\|(I-\Pi)W^*\|_F^{2}=0.
\]
That is, \(\Pi W^*=W^*\). Namely,
\[
\Bigl(I-\frac{1}{K}\mathbf{1}\mathbf{1}^{\top}\Bigr)W^*=W^*
\Longleftrightarrow
-\frac{1}{K}\mathbf{1}\mathbf{1}^{\top}W^*=0.
\]

Equivalently,
\[
\mathbf{1}\mathbf{1}^{\top}W^*=0.
\]

Moreover, \(\mathbf{1}\mathbf{1}^{\top}W^*\) can be written as
\[
\mathbf{1}\mathbf{1}^{\top}W^*=\mathbf{1}(\mathbf{1}^{\top}W^*).
\]

The left-hand side is a \(K\times d\) matrix, whose every row equals the same \(1\times d\) vector \(\mathbf{1}^{\top}W^*\).

Therefore,
\[
\mathbf{1}(\mathbf{1}^{\top}W^*)=0
\Longleftrightarrow
\mathbf{1}^{\top}W^*=0.
\]

Conversely, if \(\mathbf{1}^{\top}W^*=0\) then \(\mathbf{1}\mathbf{1}^{\top}W^*=0\), and substituting back immediately yields \(\Pi W^*=W^*\). Hence,
\[
\Pi W^*=W^*
\Longleftrightarrow
\mathbf{1}^{\top}W^*=0.
\]
The for 
\(W\in\mathbb{R}^{K\times d}
\),
write \(W\in\mathbb{R}^{K\times d}\) in block-row form as
\(W=\begin{bmatrix} w_1^{\top}\\ \vdots\\ w_K^{\top}\end{bmatrix}.\)

Left-multiplying by \(\mathbf{1}^{\top}\) yields
\[
\mathbf{1}^{\top}W=\sum_{k=1}^{K}w_k^{\top}.
\]

Therefore,
\[
\mathbf{1}^{\top}W=0
\Longleftrightarrow
\sum_{k=1}^{K}w_k^{\top}=0.
\]
That is,
\[
\sum_{k=1}^{K}w_k=0.
\]

In summary,
\[
\sum_{k=1}^{K}w_k=0
\]
and equivalently,
\[
\Pi W=W.
\]

 (ii) Proof of statement \ref{state2}.
 
Our global objective is
\[
f(W,H)=g(WH)+\frac{\lambda_W}{2}\|W\|_F^2+\frac{\lambda_H}{2}\|H\|_F^2,
\]
where
\[
g(WH)=\frac{1}{N}\sum_{m=1}^{M}\sum_{|S|=m}\sum_{i=1}^{r_{m,S}}\ell(Wh_{m,S,i},y_S).
\]

If \((W,H)\) is a global minimizer of \(f\), then for any fixed \(W\) at its current value, it must also minimize \(f\) over the variable \(H\); similarly, fixing \(H\) and optimizing over \(W\). Otherwise, if there exists \(H'\) such that \(f(W,H')<f(W,H)\), then \((W,H)\) cannot be a global minimizer, a contradiction.

Fix some \(m,S\). In \(H\), keep all variables fixed except \(h_{m,S,i}\). Therefore, the sectional objective for this variable can be written as
\[
\Phi_{m,S,i}(h)=\frac{1}{N}\ell(Wh,y_S)+\frac{\lambda_H}{2}\|h\|_2^2.
\]

Moreover, since the cross-entropy loss \(\ell(z,y_S)\) is convex in \(z\), the mapping \(h\mapsto \ell(Wh,y_S)\) is also convex (a convex function composed with a linear map is convex). Together with the term \(\frac{\lambda_H}{2}\|h\|_2^2\), the function \(\Phi_{m,S,i}(h)\) is \(\lambda_H\)-strongly convex, and hence has a unique minimizer.

Thus there exists a unique
\[
\hat{h}_{m,S}(W)=\arg\min_{h\in\mathbb{R}^d}\Bigl[\frac{1}{N}\ell(Wh,y_S)+\frac{\lambda_H}{2}\|h\|_2^2\Bigr].
\]
Moreover, this minimizer depends only on \(W,S\) and the loss form, and does not depend on \(i\).

Since \(\Phi_{m,S,i}\) is the same function for different \(i\) within the same \((m,S)\), for all \(i=1,\cdots,r_{m,S}\), we have
\[
h_{m,S,i}=\hat{h}_{m,S}(W).
\]

Hence,
\[
h_{m,S,1}=h_{m,S,2}=\cdots=h_{m,S,r_{m,S}}.
\]

That is,
\begin{equation}
    \label{eq:within-group-collapse}
    h_{m,S,i}=h_{m,S}\quad(\forall i).
\end{equation}

(iii) Proof of statement~\ref{state3}.

By statement~\ref{state2}, within-group collapse has already been established. Hence, for each \(k\in[K]\), there exists a multiplicity-one prototype \(h_{1,k}:=h_{1,\{k\}}\) such that \(h_{1,\{k\},i}=h_{1,k}\) for all \(i=1,\dots,r_{1,k}\).

We use the tightness condition of the PAL affine lower bound in Lemma~\ref{lem:your-lemma} at multiplicity one. Fix \(k\in[K]\), and define \(z^{(k)}:=Wh_{1,k}\). For \(m=1\) and \(S=\{k\}\), the tightness condition gives two scalars \(z_{\mathrm{in}}^{(k)},z_{\mathrm{out}}^{(k)}\in\mathbb R\) such that
\begin{equation}
\label{eq:state3-two-level-logits}
z_k^{(k)}=z_{\mathrm{in}}^{(k)},\qquad
z_\ell^{(k)}=z_{\mathrm{out}}^{(k)}\ (\ell\ne k),
\qquad
z_{\mathrm{in}}^{(k)}-z_{\mathrm{out}}^{(k)}
=
\Delta_1,
\end{equation}
where \(\Delta_1:=\log((K-1)c_{1,1})\).

By statement~\ref{state1}, \(W^\top\mathbf 1=0\). Hence \(\sum_{\ell=1}^K z_\ell^{(k)}=\mathbf 1^\top Wh_{1,k}=0\). Combining this with \eqref{eq:state3-two-level-logits}, we obtain
\begin{equation}
\label{eq:state3-centered-two-level}
z_{\mathrm{in}}^{(k)}=\frac{K-1}{K}\Delta_1,
\qquad
z_{\mathrm{out}}^{(k)}=-\frac{1}{K}\Delta_1.
\end{equation}
Thus the two tight multiplicity-one logit values depend only on \(c_{1,1}\), and not on the class index \(k\).

Next define \(\phi_k(z):=\mathcal L(z,y_{\{k\}})\). The loss \(\phi_k\) is invariant under any permutation of the coordinates in \([K]\setminus\{k\}\), and \(z^{(k)}\) in \eqref{eq:state3-two-level-logits} is also invariant under such permutations. Hence all out-of-class gradient coordinates are equal: there exists \(q_{\mathrm{out}}^{(k)}\) such that \((\nabla\phi_k(z^{(k)}))_\ell=q_{\mathrm{out}}^{(k)}\) for all \(\ell\ne k\). Let \(q_{\mathrm{in}}^{(k)}:=(\nabla\phi_k(z^{(k)}))_k\).

By the shift-invariance of the PAL loss in Lemma~\ref{lem:S6_shift_invariance}, \(\phi_k(z+c\mathbf 1)=\phi_k(z)\) for every \(c\in\mathbb R\). Differentiating with respect to \(c\) at \(c=0\) yields \(\mathbf 1^\top\nabla\phi_k(z^{(k)})=0\). Therefore, \(q_{\mathrm{in}}^{(k)}+(K-1)q_{\mathrm{out}}^{(k)}=0\), and the gradient has the form
\begin{equation}
\label{eq:state3-gradient-direction}
\nabla_z\mathcal L(z^{(k)},y_{\{k\}})
=
q_{\mathrm{out}}^{(k)}(\mathbf 1-Ke_k).
\end{equation}

Since the two values in \eqref{eq:state3-centered-two-level} are independent of \(k\), and the PAL loss is permutation symmetric, the scalar \(q_{\mathrm{out}}^{(k)}\) is independent of \(k\). Denote this common scalar by \(\alpha_1\). For multiplicity one, PAL reduces to the usual single-label cross-entropy, so one may compute
\[
\alpha_1
=
\frac{1}{e^{\Delta_1}+K-1}
=
\frac{1}{(K-1)(1+c_{1,1})}
>0.
\]
Thus \eqref{eq:state3-gradient-direction} becomes
\begin{equation}
\label{eq:state3-gradient-alpha}
\nabla_z\mathcal L(z^{(k)},y_{\{k\}})
=
\alpha_1(\mathbf 1-Ke_k),
\qquad
\forall k\in[K].
\end{equation}

Now consider the first-order optimality condition with respect to the collapsed prototype \(h_{1,k}\). The contribution of group \((1,\{k\})\) to the objective is
\[
\frac{r_{1,k}}{N}\mathcal L(Wh_{1,k},y_{\{k\}})
+
\frac{\lambda_H}{2}r_{1,k}\|h_{1,k}\|_2^2.
\]
Taking the derivative with respect to \(h_{1,k}\), using \(r_{1,k}>0\), and setting it to zero gives
\begin{equation}
\label{eq:state3-foc-h}
h_{1,k}
=
-\frac{1}{\lambda_HN}
W^\top\nabla_z\mathcal L(z^{(k)},y_{\{k\}}).
\end{equation}
Substituting \eqref{eq:state3-gradient-alpha} into \eqref{eq:state3-foc-h}, and using \(W^\top\mathbf 1=0\) from statement~\ref{state1}, we obtain
\begin{equation}
\label{eq:state3-raw-duality}
h_{1,k}
=
\frac{\alpha_1K}{\lambda_HN}w_k.
\end{equation}
Define \(C_1:=\alpha_1K/(\lambda_HN)>0\). Then \eqref{eq:state3-raw-duality} gives \(h_{1,k}=C_1w_k\) for all \(k\in[K]\).

Summing \eqref{eq:state3-raw-duality} over \(k=1,\dots,K\), and using \(\sum_{k=1}^K w_k=0\), we get \(\sum_{k=1}^K h_{1,k}=0\). Therefore,
\begin{equation}
\label{eq:state3-centered-duality}
\widehat h_{1,k}
=
h_{1,k}
-
\frac{1}{K}\sum_{\ell=1}^K h_{1,\ell}
=
h_{1,k}
=
C_1w_k.
\end{equation}
This proves statement~\ref{state3}.

(iv) Proof of statement \ref{state4}.

By statement~\ref{state3}, there exists a constant $C_1>0$ such that
\[
\widehat h_{1,k}=C_1 w_k,
\qquad
\forall k\in[K].
\]
Now define
\[
u_k:=\frac{w_k}{\sqrt{r_{1,k}}},
\qquad k\in[K].
\]
Equivalently,
\[
w_k=\sqrt{r_{1,k}}\,u_k.
\]
Substituting this identity into the conclusion of statement~\ref{state3}, we obtain
\[
\widehat h_{1,k}
=
C_1\sqrt{r_{1,k}}\,u_k,
\qquad
\forall k\in[K].
\]
This proves statement~\ref{state4}.

(v) Proof of statement \ref{state5}.

Fix any \(m\in\{2,\dots,M\}\) and any label set \(S\subset [K]\) with \(|S|=m\).
For each fixed \(S\), let \(z^{(S)}:=Wh_{m,S}\) denote the collapsed logits of this group.
By Lemma~\ref{lem:your-lemma}, there exist \(z_{\mathrm{in}}^{(S)},z_{\mathrm{out}}^{(S)}\in\mathbb{R}\) such that
\[
z_i^{(S)}=z_{\mathrm{in}}^{(S)} \ (\forall i\in S),\qquad
z_j^{(S)}=z_{\mathrm{out}}^{(S)} \ (\forall j\notin S),
\]
and the gap satisfies
\[
z_{\mathrm{in}}^{(S)}-z_{\mathrm{out}}^{(S)}
=
\log\!\left(\frac{K-m}{m}c_{1,m}\right).
\]

For brevity, define
\[
\Delta_m:=\log\!\left(\frac{K-m}{m}c_{1,m}\right),
\]
which depends only on \(m\) and \(c_{1,m}\). Then the three conditions above can be written as
\begin{equation}
\label{eq:q21}
\begin{cases}
z_i^{(S)} = z_{\mathrm{in}}^{(S)} \ (\forall i \in S),\\[3pt]
z_j^{(S)} = z_{\mathrm{out}}^{(S)} \ (\forall j \notin S),\\[3pt]
z_{\mathrm{in}}^{(S)} - z_{\mathrm{out}}^{(S)} = \Delta_m.
\end{cases}
\end{equation}

From Result \ref{state1} of Lemma~\ref{lem3}, we already have
\[
\sum_{k=1}^K w_k=0.
\]
Thus, for any \(h\in\mathbb{R}^d\),
\[
\sum_{k=1}^K (Wh)_k
=
\sum_{k=1}^K w_k^\top h
=
\left(\sum_{k=1}^K w_k\right)^\top h
=
0.
\]
In particular, for \(h=h_{m,S}\), this gives
\begin{equation}
\label{eq:q22}
\sum_{k=1}^K z_k^{(S)}=0.
\end{equation}

Combining \eqref{eq:q21} and \eqref{eq:q22}, we obtain
\begin{equation}
\label{eq:q23}
m z_{\mathrm{in}}^{(S)} + (K-m) z_{\mathrm{out}}^{(S)} = 0.
\end{equation}

Combining the third line of \eqref{eq:q21} with \eqref{eq:q23}, we solve
\[
z_{\mathrm{in}}^{(S)}=\frac{K-m}{K}\Delta_m,
\qquad
z_{\mathrm{out}}^{(S)}=-\frac{m}{K}\Delta_m.
\]
Hence, for all choices of \(S\), we have
\[
w_i^\top h_{m,S}=\frac{K-m}{K}\Delta_m \ (\forall i\in S),
\qquad
w_j^\top h_{m,S}=-\frac{m}{K}\Delta_m \ (\forall j\notin S).
\]

For each \(S\), due to within-group collapse, the contribution of group \((m,S)\) to the objective is
\[
\frac{r_{m,S}}{N}\mathcal{L}(Wh_{m,S},y_S)
+
\frac{\lambda_H}{2}r_{m,S}\|h_{m,S}\|_2^2.
\]
Taking the derivative with respect to \(h_{m,S}\) and setting it to zero gives
\[
0
=
\frac{r_{m,S}}{N}W^\top \nabla_z \mathcal{L}(z^{(S)},y_S)
+
\lambda_H r_{m,S} h_{m,S}.
\]
Rearranging, we obtain
\begin{equation}
\label{eq:q32}
h_{m,S}
=
-\frac{1}{\lambda_H N}W^\top \nabla_z \mathcal{L}(z^{(S)},y_S).
\end{equation}

Next we identify the structure of \(\nabla_z \mathcal{L}(z^{(S)},y_S)\).
Let \(\phi(z):=\mathcal{L}(z,y_S)\) for \(z\in\mathbb{R}^K\).
Denote by \(P\) any permutation matrix that permutes coordinates in \(S\) and fixes \(S^c\), and by \(\Omega\) any permutation matrix that permutes coordinates in \(S^c\) and fixes \(S\).
Then \(\phi(Pz)=\phi(z)\) and \(\phi(\Omega z)=\phi(z)\).

Since \(\phi\) is differentiable, taking gradients in \(\phi(Pz)=\phi(z)\) yields
\begin{equation}
\label{eq:q24}
\nabla \phi(Pz)=P\nabla \phi(z).
\end{equation}

By the structure in \eqref{eq:q21}, there exist scalars \(z_{\mathrm{in}},z_{\mathrm{out}}\) such that
\(z_i=z_{\mathrm{in}}\) for all \(i\in S\) and \(z_j=z_{\mathrm{out}}\) for all \(j\notin S\).
Hence, for any permutation \(P\) acting only on \(S\), we have
\begin{equation}
\label{eq:q25}
Pz=z.
\end{equation}
Substituting \eqref{eq:q25} into \eqref{eq:q24}, we obtain
\begin{equation}
\label{eq:q26}
\nabla \phi(z)=\nabla \phi(Pz)=P\nabla \phi(z).
\end{equation}

Therefore, for any \(i,j\in S\), taking \(P_{ij}\) to be the permutation matrix that swaps the \(i\)-th and \(j\)-th coordinates,
we obtain from \eqref{eq:q26} that
\[
(\nabla \phi(z))_i=(\nabla \phi(z))_j.
\]
Thus, there exists a scalar \(q_{\mathrm{in}}^{(S)}\) such that
\begin{equation}
\label{eq:q27}
(\nabla \phi(z))_i=q_{\mathrm{in}}^{(S)} \qquad (\forall i\in S).
\end{equation}

Similarly, there exists a scalar \(q_{\mathrm{out}}^{(S)}\) such that
\begin{equation}
\label{eq:q28}
(\nabla \phi(z))_j=q_{\mathrm{out}}^{(S)} \qquad (\forall j\notin S).
\end{equation}

Additionally, by Lemma~\ref{lem:S6_shift_invariance}, for every \(c\in\mathbb{R}\),
\[
\mathcal{L}(z+c\mathbf{1},y_S)=\mathcal{L}(z,y_S).
\]
In our notation, this means \(\phi(z+c\mathbf{1})=\phi(z)\).
Treat this as a univariate function \(\psi(c):=\phi(z+c\mathbf{1})\). Since \(\psi\) is constant, \(\psi'(0)=0\).
By the chain rule,
\[
\psi'(0)=\langle \nabla \phi(z),\mathbf{1}\rangle=\mathbf{1}^\top \nabla \phi(z).
\]
Hence
\begin{equation}
\label{eq:q29}
\mathbf{1}^\top \nabla \phi(z)=0.
\end{equation}

Combining \eqref{eq:q27}, \eqref{eq:q28}, and \eqref{eq:q29}, we obtain
\begin{equation}
\label{eq:q30}
m q_{\mathrm{in}}^{(S)}+(K-m) q_{\mathrm{out}}^{(S)}=0.
\end{equation}

Therefore, by the permutation symmetry in \eqref{eq:q27}, \eqref{eq:q28}, and the orthogonality condition \eqref{eq:q29},
the gradient must lie in the one-dimensional subspace spanned by the label direction.
Hence there exists a scalar \(\alpha_m\) (depending only on \(m\) and \(c_{1,m}\), and not on \(S\)) such that
\begin{equation}
\label{eq:q31}
\nabla_z \mathcal{L}(z^{(S)},y_S)
=
\alpha_m\left(\mathbf{1}-\frac{K}{m}I_S\right).
\end{equation}

Substituting \eqref{eq:q31} into \eqref{eq:q32}, we obtain
\begin{equation}
\label{eq:q33}
h_{m,S}
=
-\frac{\alpha_m}{\lambda_H N}
W^\top\left(\mathbf{1}-\frac{K}{m}I_S\right).
\end{equation}

From Result \ref{state1} of Lemma~\ref{lem3}, we have \(W^\top \mathbf{1}=0\). Substituting this into \eqref{eq:q33} yields
\[
W^\top\left(\mathbf{1}-\frac{K}{m}I_S\right)
=
-\frac{K}{m}W^\top I_S
=
-\frac{K}{m}\sum_{k\in S} w_k.
\]
Therefore,
\[
h_{m,S}
=
\frac{\alpha_m K}{m\lambda_H N}\sum_{k\in S} w_k.
\]
This gives
\[
h_{m,S}
=
\widehat C_m \sum_{k\in S} w_k,
\qquad
\widehat C_m:=\frac{\alpha_m K}{m\lambda_H N}.
\]

Finally, by the definition \(w_k=\sqrt{r_{1,k}}\,u_k\), and taking \(C_m:=\widehat C_m\), we conclude that
\[
h_{m,S}
=
C_m\sum_{k\in S}\sqrt{r_{1,k}}\,u_k.
\]
This proves statement~\ref{state5}.

\end{proof}

\section{Technical Spectral Lemmas for the Spectral Route}\label{app:2}

This appendix collects the technical lemmas that power the spectral route used throughout the main proofs.
The core idea is to funnel all label-imbalance information into the centered label second-moment operator
and to quantify its effective curvature on the centered subspace via the spectral constant.


\begin{lemma}[Matrix-form spectral lower bound]
\label{lem:s3}
Let $B_m:=\Pi G_m\Pi$. Under Assumption~\ref{assump:nondeg-centered}, for any $Z\in\mathbb{R}^{K\times d}$,
\[
\mathrm{Tr}\!\left(Z^\top B_m Z\right)\;\ge\;\kappa_m\|\Pi Z\|_F^2.
\]
Equivalently, writing $Z=[z_1,\dots,z_d]$ with $z_t\in\mathbb{R}^K$,
\[
\sum_{t=1}^d z_t^\top B_m z_t \;\ge\;\kappa_m\sum_{t=1}^d \|\Pi z_t\|_2^2.
\]
\end{lemma}

\begin{proof}
Fix any column $z_t$. Since $B_m=\Pi B_m\Pi$ and $\Pi^2=\Pi$, we have
\[
z_t^\top B_m z_t
= z_t^\top \Pi B_m \Pi z_t
= (\Pi z_t)^\top B_m (\Pi z_t).
\]
Moreover, $\Pi z_t\in \mathrm{range}(\Pi)$. Therefore, by Assumption~\ref{assump:nondeg-centered},
\[
(\Pi z_t)^\top B_m (\Pi z_t)\;\ge\;\kappa_m \|\Pi z_t\|_2^2.
\]
Summing over $t=1,\dots,d$ yields
\[
\mathrm{Tr}(Z^\top B_m Z)=\sum_{t=1}^d z_t^\top B_m z_t \ge \kappa_m\sum_{t=1}^d \|\Pi z_t\|_2^2
= \kappa_m\|\Pi Z\|_F^2,
\]
which completes the proof.
\end{proof}


\begin{lemma}[A trace inequality]\label{lem:trace-ineq}
Assume $B\in\mathbb{R}^{K\times K}$ is positive semidefinite, i.e.,
\begin{equation}\ 
B\succeq 0.
\end{equation}
Let $\Theta\in\mathbb{R}^{K\times d}$ be arbitrary. Then
\begin{equation} 
\operatorname{Tr}\!\bigl(\Theta^{\top}B\Theta\bigr)\ \le\ \operatorname{Tr}(B)\,\|\Theta\|_{F}^{2}.
\end{equation}
\end{lemma}

\begin{proof}
Since $B\succeq 0$ and $B$ is symmetric, there exist an orthogonal matrix
$U\in\mathbb{R}^{K\times K}$ and a nonnegative diagonal matrix
$\Lambda=\operatorname{diag}(\lambda_1,\ldots,\lambda_K)$ with $\lambda_i\ge 0$ such that
\[
B=U\Lambda U^{\top}.
\]
Define
\[ 
\widehat{\Theta}:=U^{\top}\Theta\in\mathbb{R}^{K\times d}.
\]
Then
\[
\operatorname{Tr}(\Theta^{\top}B\Theta)
=
\operatorname{Tr}\!\bigl(\Theta^{\top}U\Lambda U^{\top}\Theta\bigr)
=
\operatorname{Tr}\!\bigl(\widehat{\Theta}^{\top}\Lambda\widehat{\Theta}\bigr).
\]
Write $\widehat{\Theta}$ by rows as
\[
\widehat{\Theta}=
\begin{bmatrix}
\widehat{\theta}_1^{\top}\\
\vdots\\
\widehat{\theta}_K^{\top}
\end{bmatrix},
\qquad
\widehat{\theta}_i\in\mathbb{R}^{d}.
\]
Then
\begin{equation}\label{eq:trace-ineq:S8}
\operatorname{Tr}\!\bigl(\widehat{\Theta}^{\top}\Lambda\widehat{\Theta}\bigr)
=
\sum_{i=1}^{K}\lambda_i\|\widehat{\theta}_i\|^{2}.
\end{equation}
On the other hand, since $U$ is orthogonal,
\begin{equation}\label{eq:trace-ineq:S9}
\|\Theta\|_{F}^{2}
=
\|U^{\top}\Theta\|_{F}^{2}
=
\|\widehat{\Theta}\|_{F}^{2}
=
\sum_{i=1}^{K}\|\widehat{\theta}_i\|^{2}.
\end{equation}
Moreover,
\begin{equation}\label{eq:trace-ineq:S10}
\operatorname{Tr}(B)
=
\operatorname{Tr}(U\Lambda U^{\top})
=
\operatorname{Tr}(\Lambda)
=
\sum_{i=1}^{K}\lambda_i.
\end{equation}
Since $\lambda_i\ge 0$, from \eqref{eq:trace-ineq:S8} we obtain
\[
\operatorname{Tr}(\Theta^{\top}B\Theta)
=
\sum_{i=1}^{K}\lambda_i\|\widehat{\theta}_i\|^{2}
\le
\left(\sum_{i=1}^{K}\lambda_i\right)\left(\sum_{i=1}^{K}\|\widehat{\theta}_i\|^{2}\right)
=
\operatorname{Tr}(B)\cdot\|\Theta\|_{F}^{2},
\]
where the last equality uses \eqref{eq:trace-ineq:S9}--\eqref{eq:trace-ineq:S10}.

\end{proof}


\begin{lemma}[$\Theta$--$H$ interface inequality]\label{lem:S5_theta_H_interface}
Fix
\(
m \in \{1,\dots,M\}.
\)
For any label subset
\(
S \subset [K]\quad \text{with}\quad |S|=m,
\)
let the corresponding sample count be
\(
r_{m,S}\ge 0.
\)
Define
\[
N_m := \sum_{|S|=m} r_{m,S},
\qquad
N_m^j := \sum_{|S|=m} r_{m,S}\, I_S(j)\quad (j\in[K]).
\]
Without loss of generality, we assume that for each multiplicity level \(m\), for each label \(j\), the number of times label \(j\) appears in the multiplicity-\(m\) group is strictly positive, i.e.,
    \[
    N_m^j > 0, \qquad \forall\, m \in \{1,\dots,M\}, \quad \forall\, j \in \{1,\dots,K\}.
    \]

For each sample $(m,S,i)$, assign a feature vector
\(
h_{m,S,i}\in\mathbb{R}^d
\qquad (i=1,\dots,r_{m,S}).
\)
Define the global mean and the label-weighted means by
\[
\bar h_m := \frac{1}{N_m}\sum_{|S|=m}\sum_{i=1}^{r_{m,S}} h_{m,S,i},
\qquad
\bar h_m^{\,j} := \frac{1}{N_m}\sum_{|S|=m}\sum_{i=1}^{r_{m,S}} I_S(j)\,h_{m,S,i}.
\]
Define the matrix $\Theta_m\in\mathbb{R}^{K\times d}$ row-wise by
\[
\Theta_m(j,:) := \Bigl(\bar h_m-\frac{K}{m}\bar h_m^{\,j}\Bigr)^{\top},
\qquad j=1,\dots,K.
\]
Then we have the Frobenius-norm bound
\[
\|\Theta_m\|_F^{2}
\ \le\
\Biggl[
\frac{2K}{N_m}
+\frac{2K^2}{m^2}\Biggl(
\max_{|S|=m}\sum_{j=1}^{K}\frac{I_S(j)}{N_m^{j}}
\Biggr)
\Biggr]
\sum_{|S|=m}\sum_{i=1}^{r_{m,S}}\|h_{m,S,i}\|^{2}.
\]
\end{lemma}

\begin{proof}

Using
\begin{equation}
    \label{eq:S5_ab_ineq}
\|a-b\|^2\le 2\|a\|^2+2\|b\|^2,
\end{equation}
by the definition of $\Theta_m$ ,
\(\
\|\Theta_m\|_F^2
=
\sum_{j=1}^{K}\left\|\bar h_m-\frac{K}{m}\bar h_m^{\,j}\right\|^2.
\)
Applying \eqref{eq:S5_ab_ineq} to each term yields
\begin{equation}\label{eq:S5_step1_bound}
\|\Theta_m\|_F^2
\le
2K\|\bar h_m\|^2+\frac{2K^2}{m^2}\sum_{j=1}^{K}\|\bar h_m^{\,j}\|^2.
\end{equation}

Since $\bar h_m$ is the average of $N_m$ vectors,
Jensen's inequality gives
\begin{equation}\label{eq:S5_jensen_global}
\|\bar h_m\|^2
=
\left\|
\frac{1}{N_m}\sum_{|S|=m}\sum_{i=1}^{r_{m,S}} h_{m,S,i}
\right\|^2
\le
\frac{1}{N_m}\sum_{|S|=m}\sum_{i=1}^{r_{m,S}}\|h_{m,S,i}\|^2.
\end{equation}

For each \(j\in[K]\), by the definition of \(\bar h_m^{\,j}\),
\[
\bar h_m^{\,j}
=
\frac{1}{N_m}
\sum_{|S|=m}\sum_{i=1}^{r_{m,S}} I_S(j)\,h_{m,S,i}.
\]
Hence,
\[
\|\bar h_m^{\,j}\|^2
=
\left\|
\frac{1}{N_m}
\sum_{|S|=m}\sum_{i=1}^{r_{m,S}} I_S(j)\,h_{m,S,i}
\right\|^2.
\]
The sum above contains \(N_m^j\) nonzero terms counted with multiplicity. By Jensen's inequality,
\[
\left\|
\frac{1}{N_m^j}
\sum_{|S|=m}\sum_{i=1}^{r_{m,S}} I_S(j)\,h_{m,S,i}
\right\|^2
\le
\frac{1}{N_m^j}
\sum_{|S|=m}\sum_{i=1}^{r_{m,S}} I_S(j)\,\|h_{m,S,i}\|^2.
\]
Since
\[
\bar h_m^{\,j}
=
\frac{N_m^j}{N_m}
\left(
\frac{1}{N_m^j}
\sum_{|S|=m}\sum_{i=1}^{r_{m,S}} I_S(j)\,h_{m,S,i}
\right),
\]
we obtain
\[
\|\bar h_m^{\,j}\|^2
\le
\frac{(N_m^j)^2}{N_m^2}
\cdot
\frac{1}{N_m^j}
\sum_{|S|=m}\sum_{i=1}^{r_{m,S}} I_S(j)\,\|h_{m,S,i}\|^2
=
\frac{N_m^j}{N_m^2}
\sum_{|S|=m}\sum_{i=1}^{r_{m,S}} I_S(j)\,\|h_{m,S,i}\|^2.
\]
Because \(0<N_m^j\le N_m\), we have
\[
\frac{N_m^j}{N_m^2}\le \frac{1}{N_m^j}.
\]
Therefore,
\[
\|\bar h_m^{\,j}\|^2
\le
\frac{1}{N_m^j}
\sum_{|S|=m}\sum_{i=1}^{r_{m,S}} I_S(j)\,\|h_{m,S,i}\|^2.
\]
Summing over \(j\) and exchanging the order of summation, we get
\[
\sum_{j=1}^{K}\|\bar h_m^{\,j}\|^2
\le
\sum_{|S|=m}\sum_{i=1}^{r_{m,S}}\|h_{m,S,i}\|^2
\left(
\sum_{j=1}^{K}\frac{I_S(j)}{N_m^{j}}
\right).
\]
Using
\[
\sum_{j=1}^{K}\frac{I_S(j)}{N_m^{j}}
\le
\max_{|S|=m}\sum_{j=1}^{K}\frac{I_S(j)}{N_m^{j}},
\]
we obtain
\begin{equation}\label{eq:S5_sum_conditional_bound_max}
\sum_{j=1}^{K}\|\bar h_m^{\,j}\|^2
\le
\left(
\max_{|S|=m}\sum_{j=1}^{K}\frac{I_S(j)}{N_m^{j}}
\right)
\sum_{|S|=m}\sum_{i=1}^{r_{m,S}}\|h_{m,S,i}\|^2.
\end{equation}

Finally, substituting \eqref{eq:S5_jensen_global} and \eqref{eq:S5_sum_conditional_bound_max} into
\eqref{eq:S5_step1_bound} yields 
\[
\|\Theta_m\|_F^{2}
\ \le\
\Biggl[
\frac{2K}{N_m}
+\frac{2K^2}{m^2}\Biggl(
\max_{|S|=m}\sum_{j=1}^{K}\frac{I_S(j)}{N_m^{j}}
\Biggr)
\Biggr]
\sum_{|S|=m}\sum_{i=1}^{r_{m,S}}\|h_{m,S,i}\|^{2}.
\]
\end{proof}


\begin{remark}\label{rem:S5_beta_pi}
Consider the quantity
\begin{equation}\label{eq:S5_quantity_interest}
\max_{|S|=m}\sum_{j=1}^{K}\frac{I_S(j)}{N_m^{j}}.
\end{equation}
It is purely determined by the data statistics $\{r_{m,S}\}$ (equivalently, by $N_m$, $N_m^j$) and by $m,K$,
and does not depend on the model parameters (such as $W,H$) or the training procedure.

For convenience, define
\[
\beta_m:=\max_{|S|=m}\sum_{j=1}^{K}\frac{I_S(j)}{N_m^{j}}, 
\qquad
\bar\pi_m^{\,j}:=\frac{N_m^{j}}{N_m},
\qquad
\pi_{m,\min}:=\min_{j}\bar\pi_m^{\,j}.
\]

Then
\[
\sum_{j\in S}\frac{1}{N_m^{j}}
=
\frac{1}{N_m}\sum_{j\in S}\frac{1}{\bar\pi_m^{\,j}}
\le
\frac{1}{N_m}\cdot \frac{m}{\pi_{m,\min}}
=
\frac{m}{N_m\pi_{m,\min}},
\]
and hence
\(
\beta_m\le \frac{m}{N_m\pi_{m,\min}}.
\)
Accordingly, Lemma~\ref{lem:S5_theta_H_interface} implies the looser form
\[
\|\Theta_m\|_F^{2}
\le
\frac{2K}{N_m}\left(1+\frac{K}{m\pi_{m,\min}}\right)
\sum_{|S|=m}\sum_{i=1}^{r_{m,S}}\|h_{m,S,i}\|^{2}.
\]

That is,
\(\pi_{m,\min}
\)
is the minimum label probability. We can observe that the rarest label determines
the worst-case behavior; also, if only a very small number of labels are sparse,
then this bound can be quite loose.

\end{remark}


\begin{lemma}[Shift invariance and projection invariance]\label{lem:S6_shift_invariance}

Let
\(
\Pi \;=\; I_K \;-\; \frac{1}{K}\mathbf{1}\mathbf{1}^{\top}.
\)
For any $z\in\mathbb{R}^K$, define
\(
\mathrm{softmax}(z)_j \;:=\; \frac{e^{z_j}}{\sum_{c=1}^K e^{z_c}},
\ j\in[K]. 
\)
Define the cross-entropy loss
\(
\mathcal{L}_{\mathrm{CE}}(z,y_k)\;:=\;-\log\!\big(\mathrm{softmax}(z)_k\big), 
  k\in[K],
\)
and the pick-all-labels loss
\(
\mathcal{L}_{\mathrm{PAL}}(z,y_S)\;=\;\sum_{k\in S}\mathcal{L}_{\mathrm{CE}}(z,y_k), 
  S\subseteq[K].
\)
Given feature vectors $\{h_{m,S,i}\}\subset\mathbb{R}^d$, define the empirical risk
\[
g(WH)\;=\;\frac{1}{N}\sum_{m=1}^{M}\;\sum_{\substack{S\subseteq [K]\\ |S|=m}}\;\sum_{i=1}^{r_{m,S}}
\mathcal{L}_{\mathrm{PAL}}\big(Wh_{m,S,i},\,y_S\big).
\]
Then for any $W\in\mathbb{R}^{K\times d}$ and any feature collection $H$,
\[
g(WH)\;=\;g\big((\Pi W)H\big).
\]
\end{lemma}

\begin{proof}

For any $z\in\mathbb{R}^K$, any $c\in\mathbb{R}$, and any $j\in[K]$,
\[
\mathrm{softmax}(z+c\mathbf{1})_j
=
\frac{\exp(z_j+c)}{\sum_{c'=1}^K \exp(z_{c'}+c)}
=
\frac{\exp(c)\exp(z_j)}{\exp(c)\sum_{c'=1}^K\exp(z_{c'})}
=
\mathrm{softmax}(z)_j.
\]
Hence
\begin{equation}\label{eq:S6_softmax_shift_vec}
\mathrm{softmax}(z+c\mathbf{1})=\mathrm{softmax}(z).
\end{equation}

By \eqref{eq:S6_softmax_shift_vec}, for any $k\in[K]$,
\begin{equation}
    \label{eq:S6_shift}
\mathcal{L}_{\mathrm{CE}}(z+c\mathbf{1},y_k)=\mathcal{L}_{\mathrm{CE}}(z,y_k).
\end{equation}
Consequently, for any $S\subseteq[K]$,
\[
\mathcal{L}_{\mathrm{PAL}}(z+c\mathbf{1},y_S)=\mathcal{L}_{\mathrm{PAL}}(z,y_S).
\]

By the definition of $\Pi$, $(\Pi W)h$ differs from $Wh$ by a global shift. For any $h\in\mathbb{R}^d$, let
\(
z:=Wh, \bar z:=\frac{1}{K}\mathbf{1}^{\top}z.
\)
Then
\[
(\Pi W)h=\Pi(Wh)=\Pi z
=z-\frac{1}{K}\mathbf{1}\mathbf{1}^{\top}z
=z-\bar z\,\mathbf{1}.
\]

For each sample $(m,S,i)$, set
\(
z_{m,S,i}:=Wh_{m,S,i}.
\),
 we have
\(
(\Pi W)h_{m,S,i}=z_{m,S,i}-\bar z_{m,S,i}\mathbf{1}.
\)
Using the shift invariance \eqref{eq:S6_shift} gives
\begin{equation}\label{eq:S6_termwise_equal}
\mathcal{L}_{\mathrm{PAL}}(Wh_{m,S,i},y_S)
=\mathcal{L}_{\mathrm{PAL}}\big((\Pi W)h_{m,S,i},y_S\big).
\end{equation}
Summing \eqref{eq:S6_termwise_equal} over all $(m,S,i)$ and dividing by $N$ yields
\[
g(WH)\;=\;g\big((\Pi W)H\big).
\].
\end{proof}


\section{Standard Tool Lemmas from Prior Work}\label{app:technical-spectral}

The following two lemmas are standard tools in unconstrained-feature analyses of regularized objectives and have appeared
in closely related forms in prior work. They rely only on the chain rule and the softmax cross-entropy (PAL) structure, and
thus remain valid under our general imbalanced setting. In particular, neither lemma uses any balance assumption on the
counts $\{r_{m,S}\}$ or $\{N_m^k\}$, so they can be applied verbatim in the proofs of Theorems~\ref{thm:lower-bound-lemma2-form}
and~\ref{thm:structure-geometry}.

\begin{lemma}[Critical-point scaling identity from Lemma B.2 of \citet{Zhu2021}]\label{lem:a1}
For the regularized objective
\[
f(W,H)=g(WH)+\frac{\lambda_W}{2}\|W\|_F^2+\frac{\lambda_H}{2}\|H\|_F^2,
\]
any critical point $(W,H)$ satisfies
\[
W^\top W=\frac{\lambda_H}{\lambda_W}HH^\top,
\qquad
\|W\|_F^2=\frac{\lambda_H}{\lambda_W}\|H\|_F^2.
\]
\end{lemma}

\begin{lemma} [From Lemma C.8 of \citet{li2023neural}]\label{lem:your-lemma}
    Let \( S \subseteq \{1, \ldots, K\} \) be a subset of size \( m \) where \( 1 \leq m < K \). Then for all \( \bm{z} = (z_1, \ldots, z_K)^T \in \mathbb{R}^K \) and all \( c_{1,m} > 0 \), there exists a constant \( c_{2,m} \) such that
    \begin{equation}
        \mathcal{L}_{\mathrm{PAL}}(\bm{z}, \bm{y}_S) \geq \frac{1}{1 + c_{1,m}} \cdot \frac{m}{K - m} \cdot  \langle \mathbf{1} - \frac{K}{m}  \mathbb{I}_S, \bm{z} \rangle  + c_{2,m}.
        \label{eq:1}
    \end{equation}
    In fact, we have
    \begin{equation}
        c_{2,m} := \frac{c_{1,m} m}{c_{1,m} + 1} \log(m) + \frac{m c_{1,m}}{1 + c_{1,m}} \log \left( \frac{c_{1,m} + 1}{c_{1,m}} \right) + \frac{m}{c_{1,m} + 1} \log \left( (K - m)(c_{1,m} + 1) \right).
        \label{eq:2}
    \end{equation}
    The inequality \eqref{eq:1} is tight, i.e., achieves equality, if and only if \( \bm{z} \) satisfies all of the following:
    \begin{enumerate}
        \item For all \( i, j \in S \), we have \( z_i = z_j \) (in-group equality). Let \( z_{\mathrm{in}} \in \mathbb{R} \) denote this constant.
        \item For all \( i, j \in S^c \), we have \( z_i = z_j \) (out-group equality). Let \( z_{\mathrm{out}} \in \mathbb{R} \) denote this constant.
        \item \( z_{\mathrm{in}} - z_{\mathrm{out}} = \log \left( \frac{(K-m)}{m} c_{1,m} \right) \).
    \end{enumerate}
\end{lemma}

\end{document}